\documentclass[11pt]{article}

\usepackage[margin=0.8in]{geometry}
\usepackage{amsmath,amssymb,amsthm,mathtools}
\usepackage{bbm}
\usepackage[linesnumbered,ruled,vlined]{algorithm2e}
\usepackage{float}
\usepackage[dvipsnames]{xcolor}
\usepackage[most]{tcolorbox}
\usepackage{titletoc}
\usepackage{graphicx}
\usepackage{booktabs,tabularx,array,multirow,colortbl}
\usepackage{xspace}
\usepackage{enumitem}
\usepackage{microtype}
\usepackage{caption}
\usepackage{natbib}
\usepackage{hyperref}

\graphicspath{{figures/}{./}}

\definecolor{mainblue}{HTML}{1F5C99}
\definecolor{softblue}{HTML}{EEF6FF}
\definecolor{riskred}{HTML}{B23A48}
\definecolor{editgreen}{HTML}{167A57}
\definecolor{budgetgold}{HTML}{9A6A00}
\definecolor{darkgray}{HTML}{30343B}
\definecolor{softgray}{HTML}{F7F8FA}

\definecolor{mathblue}{HTML}{3B6FB6}
\definecolor{mmluteal}{HTML}{2A9D8F}
\definecolor{poporange}{HTML}{D97706}
\definecolor{gpqapurple}{HTML}{7E57C2}

\colorlet{mathsoft}{mathblue!8!white}
\colorlet{mmlusoft}{mmluteal!8!white}
\colorlet{popsoft}{poporange!8!white}
\colorlet{gpqasoft}{gpqapurple!8!white}

\hypersetup{
    colorlinks=true,
    hypertexnames=false,
    linkcolor=mainblue,
    citecolor=editgreen,
    urlcolor=mainblue
}

\newtheorem{theorem}{Theorem}

\SetKwInput{KwRequire}{Require}

\newtcolorbox{problembox}[1]{
    enhanced,
    breakable,
    colback=softgray,
    colframe=darkgray,
    coltitle=white,
    colbacktitle=darkgray,
    fonttitle=\bfseries,
    title={#1},
    boxrule=0.8pt,
    arc=2pt,
    left=7pt,
    right=7pt,
    top=7pt,
    bottom=7pt,
    before skip=10pt,
    after skip=10pt
}

\newtcolorbox{questionbox}{
    enhanced,
    breakable,
    colback=softgray,
    colframe=darkgray,
    boxrule=0.8pt,
    arc=2pt,
    left=8pt,
    right=8pt,
    top=8pt,
    bottom=8pt,
    before skip=10pt,
    after skip=10pt
}

\newtcolorbox{insightbox}[1]{
    enhanced,
    breakable,
    colback=softgray,
    colframe=darkgray,
    coltitle=white,
    colbacktitle=darkgray,
    fonttitle=\bfseries,
    title={#1},
    boxrule=0.7pt,
    arc=2pt,
    left=7pt,
    right=7pt,
    top=7pt,
    bottom=7pt,
    before skip=8pt,
    after skip=8pt
}

\newtcolorbox{promptouter}[1]{
    enhanced,
    breakable,
    colback=softgray,
    colframe=darkgray,
    colbacktitle=darkgray,
    coltitle=white,
    fonttitle=\bfseries,
    title={#1},
    boxrule=0.85pt,
    arc=2pt,
    left=8pt,
    right=8pt,
    top=8pt,
    bottom=8pt,
    before skip=10pt,
    after skip=12pt
}

\newtcolorbox{datasetprompt}[2][]{
    enhanced,
    breakable,
    colback=#2!5!white,
    colframe=#2,
    colbacktitle=#2!14!white,
    coltitle=black,
    fonttitle=\bfseries,
    boxrule=0.75pt,
    arc=2pt,
    left=7pt,
    right=7pt,
    top=7pt,
    bottom=7pt,
    before skip=8pt,
    after skip=8pt,
    #1
}

\newcommand{\E}{\mathbb{E}}
\newcommand{\Prob}[1]{\mathrm{Pr}\!\left[#1\right]}
\newcommand{\1}{\mathbbm{1}}
\newcommand{\cZ}{\mathcal{Z}}
\newcommand{\cD}{\mathcal{D}}

\newcommand{\ErrI}{\mathrm{Err}_{\mathrm{I}}}
\newcommand{\ErrII}{\mathrm{Err}_{\mathrm{II}}}
\newcommand{\hErrI}{\widehat{\mathrm{Err}}_{\mathrm{I}}}
\newcommand{\hErrII}{\widehat{\mathrm{Err}}_{\mathrm{II}}}
\newcommand{\EDO}{\textsc{EDO}}
\newcommand{\EDIT}{\textsc{EDIT}}

\newcommand{\EDITREP}{\textsc{EDIT-REP}}
\newcommand{\EDITsup}{\text{\textsc{EDIT}}}
\newcommand{\piEDIT}{\pi^{\EDITsup}}

\newcommand{\piEDITti}{\pi_{t,i}^{\EDITsup}}
\newcommand{\nt}{n_t}
\newcommand{\Zti}{Z_{t,i}}

\newcommand{\Ati}{A_{t,i}}
\newcommand{\sumi}{\sum_{i=1}^{\nt}}
\newcommand{\sumt}{\sum_{t=1}^{T}}
\newcommand{\scoretheta}{\widetilde\eta_{\theta}}
\newcommand{\scoreti}{\widetilde\eta_{\theta_t}(\Zti)}

\newcommand{\costti}{c(\Zti)}

\newcommand{\cX}{\mathcal{X}} 
\newcommand{\cY}{\mathcal{Y}}
\DeclareMathOperator*{\argmin}{arg\,min}

\newcommand{\EDITLLM}{\textsc{EDIT-LLM}\xspace}

\newcommand{\EDITRAND}{\textsc{EDIT-RAND}\xspace}
\newcommand{\LLMSelect}{\textsc{LLM-Select}\xspace}

\newcolumntype{Y}{>{\raggedright\arraybackslash}X}
\newcolumntype{C}[1]{>{\centering\arraybackslash}p{#1}}

\title{Edit Discovery in Human--AI Interactions}
\author{
Alireza Mirrokni\textsuperscript{1} \quad
Shayan Kiyani\textsuperscript{2} \quad
George Pappas\textsuperscript{2} \quad
Hamed Hassani\textsuperscript{2}\\[0.6em]
\small \textsuperscript{1}EPFL \qquad \textsuperscript{2}University of Pennsylvania\\
\small \texttt{seyedalireza.mirroknibanadaki@epfl.ch}\\
\small \texttt{\{shayank,pappasg,hassani\}@seas.upenn.edu}
}
\date{}

\begin{document}
\maketitle

\begin{abstract}
Human--AI interactions are becoming a rich source of data for improving AI systems: while many model responses are already satisfactory, others expose failures that could become valuable supervision if identified and corrected. We study this setting through the notion of an \emph{editor}: a more capable evaluation process that can inspect a response and either confirm it or modify it, but whose invocation is costly. The editor may be a human expert, a stronger AI system, or the same model equipped with additional computation, context, tools, or information. The key challenge is to decide which interactions to send to the editor before knowing whether an edit will actually be made, since invoking the editor itself consumes limited resources such as expert time, inference cost, latency, or tool-use overhead. We formulate this problem as \emph{Edit Discovery Optimization} (EDO), which seeks to minimize edit opportunities missed among skipped interactions while controlling unnecessary confirmations and respecting an editor-cost budget. We characterize the population-optimal policy and show that it takes the form of a cost-adjusted threshold on the posterior probability of an edit. Building on this structure, we develop EDIT (\emph{Edit Discovery via Iterative Thresholding}), an online algorithm for the partial-feedback setting in which edit labels are revealed only for selected interactions. EDIT combines an adaptive edit score, a confirmation-control threshold, and a cost-aware selection rule; it satisfies the editor budget at every round and controls the average confirmation rate up to a vanishing finite-horizon term, without distributional assumptions on the interaction stream. Across four reasoning and factual question-answering benchmarks, EDIT achieves a better trade-off between editor cost, Type-I error, and Type-II error than direct LLM-based batch selection.\footnote{Code is available at \url{https://github.com/alirezamirrokni/Edit-Discovery-in-HAI-Interactions}.}
\end{abstract}

\section{Introduction}
Data has been instrumental to the progress in modern artificial intelligence. Until recently, this largely meant collecting human- or naturally-generated content, including text, code, images, and other forms of information. However, the most readily available sources of such data have become increasingly saturated: much of the public, high-quality content has already been used, while obtaining new data of comparable quality is becoming more difficult. This calls for a rethinking of what should count as useful data for the next generation of AI systems.

One emerging source is the rapidly growing stream of human--AI interaction data, which we refer to as HAI data. In many settings, users now interact with AI systems that answer questions, handle requests, write code, summarize documents, and provide customer support. This source of data is fundamentally different from traditional supervised data. The responses are generated by the deployed system itself, and therefore are not uniformly useful as new supervision: many are already satisfactory or reflect capabilities the system has already acquired, while others expose mistakes, ambiguity, incompleteness, or misalignment with user intent. If these failure cases can be identified and corrected, however, they can be transformed from raw interaction logs into valuable supervision for improving future systems.

To formalize this correction process, we assume access to an \emph{editor}: a more capable evaluation process that, when invoked, can inspect the original interaction and either confirm the response or modify it. The editor may be a human expert, a stronger but more expensive AI system, or even the same model given additional computation, context, tools, or access to external information. Interactions on which the editor would make a modification are precisely the cases that can yield corrective supervision.

The difficulty is deciding which interactions are worth sending to the editor. Consider a large e-commerce platform such as Amazon, where customer inquiries may increasingly be handled through AI-generated responses. Such a platform could generate tens of thousands of customer--AI interactions each day. Sending every interaction to the editor would maximize the opportunity to discover useful corrections, but could require substantial resources: human experts consume time and specialized effort, stronger AI systems incur additional inference cost, and enhanced evaluation pipelines may introduce latency or tool-use overhead. At the other extreme, invoking the editor too sparingly risks missing precisely the interactions that would provide the most valuable corrective signal. The challenge is therefore to allocate limited editorial resources toward interactions that are most likely to benefit from correction, while avoiding unnecessary editor invocations on responses that would simply be confirmed. This motivates the central question of this paper:

\begin{center}
\emph{Can we selectively identify the interactions on which the editor would make an edit, while substantially reducing the cost of invoking the editor?}
\end{center}
We answer this question in two steps. First, we formulate \emph{Edit Discovery Optimization} (EDO), a constrained selection problem that formalizes the trade-off between discovering edit opportunities, avoiding unnecessary confirmations, and respecting the available editor budget. We characterize the population-optimal solution and show that it takes the form of a cost-adjusted threshold on the posterior probability of an edit. Building on this structure, we then develop EDIT (\emph{Edit Discovery via Iterative Thresholding}), a practical online algorithm for the setting in which edit outcomes are observed only for interactions that are actually sent to the editor. This partial-feedback setting creates an additional algorithmic challenge: edit labels are revealed only for selected interactions, yet the discovery errors of interest depend on both selected and skipped interactions, so in particular the policy never observes online the skipped labels that determine missed edits. The policy must therefore learn where edits are likely to occur from selectively revealed feedback while regulating its selections without access to the full set of labels. EDIT addresses this by combining an adaptive edit score with a dynamically updated confirmation-control threshold and a cost-aware selection rule.

\paragraph{Contributions.}

\begin{figure*}[t!]
    \centering
    \includegraphics[width=0.985\textwidth]{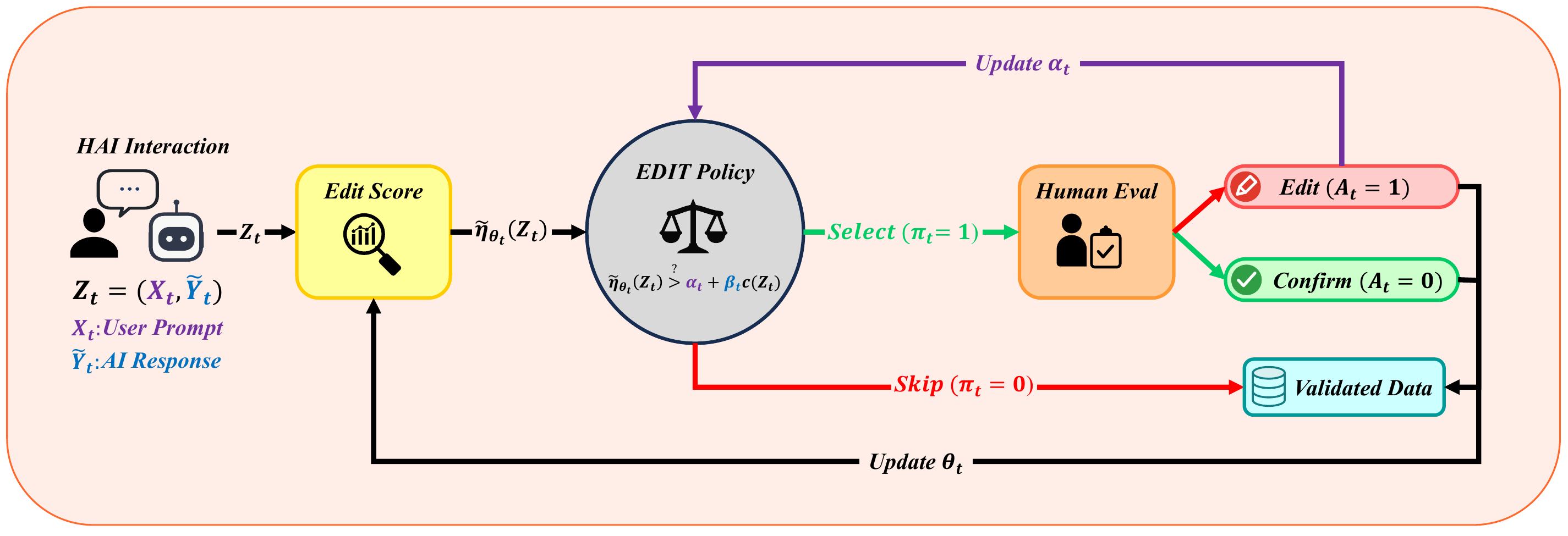}
    \caption{Overview of \EDIT{}.
Each HAI interaction \(Z_t=(X_t,\widetilde Y_t)\) receives an estimated edit score
\(\widetilde\eta_{\theta_t}(Z_t)\). Based on this score, the policy \(\pi_t\) determines whether the
interaction is sent to the editor or skipped. Selected responses are either confirmed or edited,
yielding validated data and feedback for updating the score and policy online.
    }
    \label{fig:edit-overview}
\end{figure*}


Our main contributions are as follows:
\begin{itemize}[leftmargin=1.6em,itemsep=0.35em,topsep=0.35em]
    \item We introduce Edit Discovery Optimization (EDO), a formal framework for selecting HAI
    interactions under limited editor resources. The formulation distinguishes between confirmations
    among selected interactions and missed edits among skipped interactions, while explicitly
    accounting for interaction-specific editor costs.

    \item We characterize the optimal population policy and show that it selects an interaction when
    its posterior edit probability exceeds a cost-adjusted threshold. This result identifies the
    posterior probability of editor modification as the appropriate edit score and clarifies how
    editor cost should affect the decision.

    \item Building on this characterization, we develop \EDIT{}, an online selection method that
    combines an edit score, an adaptive confirmation-control threshold, and a batchwise cost
    multiplier. We prove that \EDIT{} satisfies the editor-budget constraint at every round and
    controls the average per-round Type-I error up to a finite-horizon term that vanishes with the
    number of rounds.

    \item We evaluate \EDIT{} on four reasoning and factual question-answering benchmarks using
    Llama-3.3-70B-Instruct, with additional experiments using Qwen3-32B, response-length
    editor costs, and structured distribution shifts. The results show that informative edit scores,
    particularly the learned representation score, reduce missed edits relative to direct
    LLM-based batch selection  while maintaining reliable
    Type-I behavior and feasible editor cost.
\end{itemize}

\section{Related Work}\label{sec:related-work-main}

We summarize the closest work here and provide an extended review in Appendix~\ref{sec:related-work}.

\paragraph{Efficient evaluation and selective supervision.}
Prediction-powered inference combines model predictions with a small labeled sample to estimate population quantities with valid uncertainty \citep{angelopoulos2023ppi}, while cost-optimal active evaluation allocates an annotation budget between cheap weak raters and expensive strong raters \citep{angelopoulos2025costoptimal}. Active learning selects informative examples for model improvement \citep{settles2009active}, PAC labeling combines expert and AI labels while bounding overall dataset labeling error \citep{candes2026pac}, and learning to defer allocates prediction responsibility by accounting for the complementary capabilities of models and experts \citep{mozannar2020defer}.

\paragraph{Human--AI collaborative procedures.}
Conformal prediction sets narrow experts' choices while controlling exclusion of the correct answer \citep{straitouri2023expert}. Counterfactual extensions optimize assistance through online learning using nested sets and monotonicity of expert decisions \citep{straitouri2024counterfactual}. Collaborative uncertainty quantification prunes and augments human-proposed sets to preserve correct judgments and recover overlooked answers \citep{noorani2025collaborative, noorani2026multi}. Related frameworks use weak scores to accept, reject, or request strong verification \citep{kiyani2026verification}, and selectively acquire costly support while controlling missed beneficial assistance \citep{kiyani2026support}. Together, these approaches combine human and computational effort to improve decision quality.

\paragraph{Calibrated selection and online control.}
Conformal selection uses calibrated $p$-values to select instances whose unobserved outcomes satisfy a target criterion while controlling false discoveries \citep{jin2023selection}. Conformal Alignment extends this perspective to foundation-model outputs, selecting responses that satisfy a specified trustworthiness criterion \citep{gui2024conformalalignment}. Related extensions address interactive data analysis \citep{gui2025acs}, online conformal selection with irreversible decisions \citep{liu2025ocsarc}, and more general bounded risk criteria \citep{bai2026score}. Adaptive conformal inference uses sequential feedback to adjust prediction thresholds and regulate coverage under distribution shift \citep{gibbs2021adaptive}, while conformal risk control extends calibration beyond coverage to expected monotone losses \citep{angelopoulos2024risk}.

\section{Fundamentals of Edit Discovery}\label{sec:fundamentals}
In this section, we formalize Edit Discovery Optimization (\EDO{}) and characterize its optimal population policy. The formulation makes explicit the two ways a selection policy can fail: spending editor effort on responses ultimately confirmed without modification, or skipping responses that would have benefited from editing. It also incorporates the cost of invoking the editor under limited evaluation resources. The resulting analysis characterizes how selection should depend on edit likelihood and cost and reveals the structure that our finite-sample method seeks to emulate.

\subsection{Problem Formulation}\label{subsec:problem-formulation}
Let \(\cX\) and \(\cY\) be the prompt and answer spaces. For each prompt \(X\in\cX\), the AI system produces a response \(\widetilde Y\in\cY\). Before editor feedback, the observable item is the prompt--response pair
\[
    Z=(X,\widetilde Y)\in \cZ,
    \qquad
    \cZ:=\cX\times\cY .
\]
If the item is selected, the editor inspects the AI response and provides a validated final answer \(Y\in\cY\). When the response is unsatisfactory, this answer may correct an error, resolve ambiguity, complete missing information, or better align the response with the user's intent. Invoking the editor consumes limited resources such as expert time, annotation budget, or latency, so the goal is to select interactions most likely to yield useful corrections rather than review every response.

For each selected item, define the edit indicator as the function
\[
A(Z,Y)
:=
\begin{cases}
1, & \text{if the editor modifies the AI answer, } Y\neq \widetilde Y,\\
0, & \text{if the editor confirms the AI answer, } Y=\widetilde Y.
\end{cases}
\]
For notational brevity, we write \(A:=A(Z,Y)\) when the underlying interaction and validated answer
are clear from context. Thus, \(A=1\) marks an informative edit outcome in which the interaction reveals a
failure mode and produces corrective data that can be used, for example, as supervision for subsequent
model training. Conversely, \(A=0\) indicates that the editor accepts the AI response without modification.
Such confirmations may still be useful for auditing system behavior and for reinforcing capabilities that
the model already performs reliably.

\paragraph{Central population quantity.}
Efficient edit discovery requires predicting which AI responses are most likely to trigger an editor modification before the editor is invoked. The central population quantity is therefore the posterior edit probability
\[
    \eta(z):=\Prob{A=1\mid Z=z}.
\]
For a given interaction \(z\), this quantity measures the probability that invoking the editor will result in an edit rather than a confirmation. It therefore provides a natural measure of the expected corrective value of selecting that interaction. Responses with larger values of \(\eta(z)\) are more likely to produce corrective data, whereas responses with smaller values are more likely to be confirmed without modification. As we show in Section~\ref{subsec:population-policy}, this quantity determines the ordering of interactions under the population-optimal policy, up to an adjustment that accounts for the cost of invoking the editor.

\paragraph{Editor cost.}
Each interaction has a finite editor cost \(c(z)>0\), known at selection time. It may encode human time, monetary expense, latency, domain expertise, or any normalized resource needed to obtain feedback in practice. For instance, asking the editor to inspect a long mathematical derivation can be substantially more costly than checking a short factual answer. Allowing \(c\) to depend on \(z\) lets the policy trade off the expected value of an edit against the total effort required to obtain it.

\paragraph{Selection policy and discovery errors.}
A population selection policy is a measurable function
\(
    \pi:\cZ\to[0,1]
\),
where \(\pi(z)\) is the probability that item \(z\) is sent to the editor. Conditional on
\(Z=z\), let \(S_\pi\in\{0,1\}\) denote the realized selection decision, with
\(\Prob{S_\pi=1\mid Z=z}=\pi(z)\), generated independently of \(A\) given \(Z\). Thus
\(\pi(z)=1\) means the item is  selected, \(\pi(z)=0\) means it is  skipped, and
intermediate values allow randomization. Selecting broadly consumes scarce resources on responses that the editor may simply
confirm, whereas selecting narrowly can miss failures that would have produced valuable corrections. Edit
discovery therefore seeks to concentrate selection on likely edits while controlling both the
confirmation burden and total editor cost.

We define the Type-I error of a policy \(\pi\) as
\begin{equation}
\label{eq:population-type-I-error}
    \ErrI(\pi):=\Prob{A=0\mid S_\pi=1},
\end{equation}
the confirmation rate among selected examples. This is the probability that a selected item is already
correct and  does not require editor modification. Similarly, we define the Type-II error of
\(\pi\) as
\begin{equation}
\label{eq:population-type-II-error}
    \ErrII(\pi):=\Prob{A=1\mid S_\pi=0},
\end{equation}
the edit rate among unselected examples. This is the probability that a skipped item would have
produced an editor correction. The goal is to minimize Type-II error while controlling Type-I error and
respecting the available editor budget.

\begin{problembox}{Edit Discovery Optimization (\EDO)}
Let \(0\le \varepsilon<1\) denote the Type-I error tolerance, and let
\(B\ge0\) be the population editor budget. The population edit-discovery problem is\footnotemark
\begin{equation}
\tag{\textsc{EDO}}
\label{eq:edo}
\begin{aligned}
\underset{\pi:\cZ\to[0,1]}{\text{minimize}}
\quad &
\Prob{A=1\mid S_\pi=0}
\\
\text{subject to}\quad
&
\Prob{A=0\mid S_\pi=1}\le \varepsilon,
\\
&
\E[\pi(Z)c(Z)]\le B .
\end{aligned}
\end{equation}
\end{problembox}
\footnotetext{Throughout the population analysis, whenever a conditional probability is represented
as a ratio whose numerator and denominator are both zero, we use the convention \(0/0=0\).}
The parameter \(\varepsilon\) specifies the Type-I error tolerance, equivalently the maximum confirmation rate allowed among selected interactions.
Smaller values require the selected set to contain a higher fraction of items that actually lead to
edits, thereby placing greater emphasis on edit precision. The budget \(B\) controls the total editor cost. Thus, \EDO{} provides a  decision-theoretic formulation of selective editor invocation, making it possible to study how corrective value, selection reliability, and resource constraints jointly determine which interactions should enter the editor feedback loop.

The remainder of the analysis proceeds in two steps. Section~\ref{subsec:population-policy}
studies the idealized population regime, where the  distribution of \((Z,A)\) is known, and derives the
form of an optimal solution to \EDO{}. Section~\ref{sec:online} replaces unavailable population quantities with their observed, sequential counterparts, yielding a practical distribution-free algorithm with finite-horizon guarantees.

\subsection{Optimal Policy in the Population Regime}\label{subsec:population-policy}
\label{sec:population}
We study \EDO{} in the population regime, where the joint distribution of \((Z,A)\) is assumed known. Although unavailable in practice, this regime provides a useful benchmark for understanding the structure of the optimal policy and how posterior edit probability and editor cost should jointly determine selection. This characterization guides the method in Section~\ref{sec:online}, where these quantities are unknown and the policy must learn from partial feedback revealed only for selected interactions.

\begin{theorem}[Optimal population policy]
\label{thm:optimal-population-policy}
Let \(\pi^\star\) be an optimal policy for
\eqref{eq:edo}. Then there exist scalars \(\alpha^\star,\beta^\star\ge0\) such that
\begin{equation}
\label{eq:optimal-policy}
    \pi^\star(z)
    =
    \1\left\{
        \eta(z)> \alpha^\star+\beta^\star c(z)
    \right\},
\end{equation}
with possible randomization on the boundary
\(\{z:\eta(z)=\alpha^\star+\beta^\star c(z)\}\), if needed.
\end{theorem}

Theorem~\ref{thm:optimal-population-policy} reduces a constrained set-selection problem to a
one-dimensional ranking rule. An interaction is selected  when its probability of producing
an edit exceeds a global threshold plus a cost-dependent adjustment. The scalar \(\alpha^\star\)
controls how informative the selected set must be, with larger values filtering out interactions with
relatively low edit probability. The term \(\beta^\star c(z)\) accounts for scarce editor capacity by requiring a
higher edit probability for more expensive interactions.

This characterization also explains why learning an edit score is central. The online controller need not reconstruct the full joint distribution of \((Z,A)\), it only needs a score that orders interactions like \(\eta\) and adaptive analogues of the population thresholds. The cost term reflects scarce editor capacity, as it is active only when the budget is restrictive and vanishes otherwise.

This directly motivates the online algorithm. In practice, \(\eta\), \(\alpha^\star\), and \(\beta^\star\) are unknown, so we replace \(\eta\) with an adaptive score \(\widetilde\eta_t\), update the global threshold from observed feedback, and choose the cost multiplier each round to enforce the budget. The online policy thus preserves the structure of the population-optimal rule using quantities observable during deployment.

\section{Online Edit Discovery}
\label{sec:online}

The population analysis identifies the form of the optimal \EDO{} policy, but its ingredients are
unavailable in deployment: the interaction distribution is unknown, \(\eta\) cannot be evaluated,
and edit labels are observed only for selected responses. In this section, we address the same edit
discovery task under this information constraint, without imposing any distributional assumptions on the
incoming  stream. We replace the population errors in
Section~\ref{subsec:problem-formulation} by their online counterparts and replace the
unknown population threshold rule by an adaptive finite-sample policy. Thus, the online setting
specifies how to operationalize \EDO{} from a stream of partially revealed feedback. 

We now describe the online interaction protocol used throughout this section.

\begin{insightbox}{Online interaction protocol}
At each round \(t=1,\ldots,T\), the following sequence occurs:
\begin{enumerate}[leftmargin=1.6em,itemsep=0.25em,topsep=0.3em]
    \item \textbf{Batch arrival.} A batch
    \(\cD_t=\{\Zti\}_{i=1}^{\nt}\), with \(\nt:=|\cD_t|\), arrives together with predictable
    editor costs \(\{\costti\}_{i=1}^{\nt}\). The edit labels
    \(\Ati:=\1\{Y_{t,i}\neq\widetilde Y_{t,i}\}\) are not yet observed.
    \item \textbf{Selection.} Based on the information available at round \(t\), the policy chooses
\(\pi_t\in\{0,1\}^{\nt}\), where \(\pi_{t,i}=1\) sends
interaction \(i\) to the editor and \(\pi_{t,i}=0\) skips it.
\item \textbf{Edit feedback.} The editor returns \(Y_{t,i}\), and hence reveals
\(A_{t,i}\), only for interactions with \(\pi_{t,i}=1\). The newly revealed feedback may then be used
to update the policy before the next batch arrives.
\end{enumerate}
\end{insightbox}
We define the following two quantities as the empirical counterparts of the population Type-I and Type-II errors in Equations~\eqref{eq:population-type-I-error} and~\eqref{eq:population-type-II-error}, respectively.
For a realized stream, we define the online Type-I error as the average within-round confirmation rate,
\begin{equation}
\label{eq:empirical-type-I-error}
\hErrI(\pi)
:=
\frac{1}{T}\sumt
\frac{
\sumi
\pi_{t,i}(1-\Ati)
}{
1\lor
\sumi
\pi_{t,i}
}.
\end{equation}

Likewise, we define the online Type-II error as the average within-round edit rate among skipped interactions,
\begin{equation}
\label{eq:empirical-type-II-error}
\hErrII(\pi)
:=
\frac{1}{T}\sumt
\frac{
\sumi
(1-\pi_{t,i})\Ati
}{
1\lor
\sumi
(1-\pi_{t,i})
}.
\end{equation}
Here, we use the notation \(a\lor b:=\max\{a,b\}\). We treat each round as the relevant operational
unit and therefore average the corresponding conditional error rates across rounds rather than pooling
all interactions. This gives equal weight to different selection rounds and prevents high-volume rounds
from dominating the error criterion. \(\hErrI\) records the average confirmation rate among selected
interactions across rounds, whereas \(\hErrII\) records the average edit rate among skipped interactions.
The latter is observable only retrospectively in our fully labeled evaluation streams, while in deployment it
remains counterfactual because skipped labels are not revealed.

For the same Type-I error tolerance \(0\le\varepsilon<1\) as in \EDO{}, we seek low
\(\hErrII\) while controlling \(\hErrI\le\varepsilon\). The population budget
\(\E[\pi(Z)c(Z)]\le B\) becomes a per-round constraint,
\begin{equation}
    \sumi \pi_{t,i}\costti\le B_t,
    \qquad t=1,\ldots,T.
    \label{eq:online-budget-constraint}
\end{equation}
This provides the sequential analogue of \EDO{} under the information available in the online setting.
In particular, \(\hErrII\) depends on the unobserved labels of skipped interactions and is therefore an
evaluation criterion rather than a signal the policy can optimize directly during deployment. By contrast, the Type-I
contribution at round \(t\) depends only on selected labels, and the budget is known before selection;
both can be controlled directly by the online policy.

Theorem~\ref{thm:optimal-population-policy} suggests how to proceed without the population
quantities. We replace \(\eta\) by a parametrized edit score
\(\scoretheta(z)\in[0,1]\), where \(\theta\) denotes the parameters governing the score function,
replace \(\alpha^\star\) by an adaptive threshold \(\alpha_t\) updated from observed confirmations,
and choose a batchwise multiplier \(\beta_t\) to enforce
Equation~\eqref{eq:online-budget-constraint}. We call the resulting method \EDIT{}:
\emph{Edit Discovery via Iterative Thresholding}. When the score is learnable, \EDIT{} also updates
\(\theta\) using the previously revealed labels.

\subsection{\EDIT{} Policy}

Motivated by the population solution, \EDIT{} uses the binary threshold policy
\begin{equation}
\label{eq:online-binary-threshold-rule}
\piEDITti
=
\1\left\{
\scoreti
>
\alpha_t+\beta_t \costti
\right\},
\end{equation}
where \(\alpha_t\ge0\) is the online confirmation-control threshold and \(\beta_t\ge0\) is the
batchwise cost multiplier. Initialize \(\alpha_1=0\), and fix threshold stepsizes
\(\{\gamma_t\}_{t=1}^T\), with \(\gamma_t>0\). Given \(\alpha_t\), define
\begin{equation}
\label{eq:online-alpha-beta-definition}
\begin{aligned}
\beta_{t}
&=
\inf
\left\{
\beta\ge0:
\sum_i
\1\left\{
\scoreti
>
\alpha_t+\beta \costti
\right\}
\costti
\le B_t
\right\}, \\
\alpha_{t+1}
&=
\max\left\{
0,
\alpha_t
+
\gamma_t
\left(
\frac{
\sum_i
\pi_{t,i}^{\EDIT{}}
(1-A_{t,i})
}{
1\lor
\sum_i
\pi_{t,i}^{\EDIT{}}
}
-\varepsilon
\right)
\right\}.
\end{aligned}
\end{equation}
Note that the set inside the infimum is nonempty, since the sum is zero for all sufficiently large \(\beta\).

The batchwise \(\beta_t\) follows the same principle as the population multiplier. Given \(\alpha_t\), it is zero when the unadjusted rule is budget-feasible and otherwise increases only enough to make the selected set satisfy the current budget.

After the selected examples are evaluated and their labels are observed, the score parameter is
updated by a gradient step on the selected-sample cross-entropy loss
\begin{equation}
\label{eq:selected-cross-entropy-loss}
\ell_t(\theta)
=
\frac{
\sumi
\piEDITti
\left[
-
\Ati\log \widetilde\eta_{\theta}(\Zti)
-
(1-\Ati)\log\left(1-\widetilde\eta_{\theta}(\Zti)\right)
\right]
}{
1\lor
\sumi
\piEDITti
}.
\end{equation}
For a score learning rate \(\xi_t>0\), the score update is
\begin{equation}
\label{eq:online-theta-update-gradient}
\begin{aligned}
\theta_{t+1}
=
\theta_t
-
\xi_t
\nabla_{\theta_t} \ell_t(\theta)
=
\theta_t
-
\xi_t
\frac{
\sumi
\piEDITti
\left[
\frac{\widetilde\eta_{\theta_t}(\Zti)-\Ati}
{\widetilde\eta_{\theta_t}(\Zti)\left(1-\widetilde\eta_{\theta_t}(\Zti)\right)}
\right]
\nabla_{\theta_t} \widetilde\eta_{\theta_t}(\Zti)
}{
1\lor
\sumi
\piEDITti
}.
\end{aligned}
\end{equation}

We now show that \EDIT{} controls the average per-round Type-I error up to a vanishing term in \(T\).

\begin{theorem}[Online constraint guarantee]
\label{thm:online-binary-constraint-guarantee}
Suppose that \(\gamma_t=\gamma>0\) for all \(t=1,\dots,T\). The \EDIT{} policy satisfies
\[
\sumi
\piEDITti\costti\le B_t,
\qquad t=1,\dots,T.
\]
Moreover,
\begin{equation}
\label{eq:online-type-I-bound}
\hErrI(\piEDIT)
\le
\varepsilon
+
\frac{1+\gamma}{\gamma T}.
\end{equation}
\end{theorem}
This guarantee requires no distributional assumptions on the incoming interaction stream: batches may be dependent, non-identically distributed, or nonstationary, and the bound applies to the realized interaction and observed feedback sequence. It is also agnostic to how the edit score is constructed and therefore applies across a broad range of score designs. These types of guarantees are common in the online conformal prediction literature (e.g. look at \cite{gibbs2021adaptive, ramalingam2025relationship}). We instantiate random, representation-based, and LLM-based scores in Section~\ref{subsec:experiment-settings}. The full \EDIT{} algorithm appears in Appendix~\ref{app:algorithm}.
\section{Experiments}
\label{sec:experiments}

We evaluate \EDIT{} by Type-I error among selected interactions, Type-II error among skipped interactions, and editor cost. Section~\ref{subsec:experiment-settings} describes the scores, models, datasets, baselines, protocol, and metrics, and Section~\ref{subsec:main-results} presents the main results. Appendix~\ref{app:extended-experiments} provides experimental details, additional evaluations, and representation ablations; Appendix~\ref{app:human-audit} validates the benchmark-derived edit labels through a human audit with GPT-5.6 Sol as an independent model annotator.

\subsection{Experimental Settings}
\label{subsec:experiment-settings}

\paragraph{Edit-score design.}
Theorem~\ref{thm:online-binary-constraint-guarantee} allows any adaptive score in $[0,1]$; score quality affects efficiency, not controller validity. We consider three edit-score functions in our experiments:

\begin{itemize}[leftmargin=1.6em,itemsep=0.35em,topsep=0.3em]
    \item \textbf{Random score (\EDITRAND).}
    As an uninformative control, each interaction receives an independent score
    \(
    \widetilde\eta_t(Z_{t,i})\sim\mathrm{Uniform}(0,1)
    \).
    It retains the same \(\alpha_t\) update and budget-feasible rule as \EDIT{} but removes semantic information, isolating the contribution of the controller alone.
    
    \item \textbf{LLM score (\EDITLLM).}
    Motivated by evidence that capable language models can estimate answer correctness or uncertainty \citep{kadavath2022know}, this training-free variant asks an LLM, without the reference answer, to estimate from the prompt--response pair the probability of a human edit and return a score in \([0,1]\) for each candidate before editor invocation.

     \item \textbf{Representation score (\EDITREP).}
    Our primary score design maps the prompt--response pair
    $Z$ to a frozen representation $\phi(Z)\in\mathbb R^d$ and learns a logistic
    linear probe defined as:
    \begin{equation}
        \widetilde\eta_{\theta}(Z)
        =
        \sigma\!\left(\langle\phi(Z),\theta\rangle\right),
        \qquad
        \sigma(u)=\frac{1}{1+\exp(-u)}.
        \label{eq:online-eta-linear-probe}
    \end{equation}
    Here, only \(\theta\) is learned while \(\phi\) remains fixed, and the selected-example gradient in Eq.~\eqref{eq:online-theta-update-gradient} becomes
    \begin{equation}
        \nabla_\theta \ell_t(\theta)
        =
        \frac{
        \sum_{i=1}^{n_t}\pi_{t,i}^{\EDIT{}}
        \bigl(\widetilde\eta_\theta(Z_{t,i})-A_{t,i}\bigr)\phi(Z_{t,i})
        }{
        1\lor\sum_{i=1}^{n_t}\pi_{t,i}^{\EDIT{}}
        }.
        \label{eq:linear-probe-gradient}
    \end{equation}
    Appendix~\ref{app:embedding-size} studies larger representation models. We use all-MiniLM-L6-v2 \citep{wang2020minilm,reimers2019sbert} and update only the linear probe online. At 22.7M parameters, it is over three orders of magnitude smaller than the 70B response model, testing whether \EDITREP{} can rank edits with a lightweight encoder in the online deployment setting.
\end{itemize}

\paragraph{Models and datasets.}
The primary response model is Llama-3.3-70B-Instruct \citep{grattafiori2024llama3}, used with deterministic decoding to generate responses and, when applicable, provide the \EDITLLM{} score or execute \LLMSelect{}. Appendix~\ref{app:qwen-experiments} repeats the experiment with Qwen3-32B \citep{yang2025qwen3}. We use four complementary benchmarks: MATH-500~\citep{hendrycks2021math,lightman2024verify}, a 500-problem MATH subset; MMLU-Pro~\citep{wang2024mmlupro}, a ten-option multi-domain reasoning benchmark; PopQA~\citep{mallen2023popqa}, an entity-centric long-tail factual QA benchmark; and GPQA~\citep{rein2023gpqa}, a graduate-level multiple-choice science benchmark. 

\paragraph{Baselines.}
Our main baseline, \LLMSelect{}, receives the full indexed batch, item costs, and budget and returns a budget-feasible subset requiring correction. It gives a capable LLM joint access to all candidates for direct ranking and allocation, testing this prompted decision against \EDIT{}'s explicit edit score, Type-I feedback through $\alpha_t$, and budget control through $\beta_t$. We also compare against \EDITRAND{} and \EDITLLM{} as training-free score variants.

\paragraph{Online protocol and metrics.} We compare Type-II error at comparable editor cost and Type-I error; lower Type-II error under matched conditions indicates better edit discovery. Unless stated otherwise, runs use \(T=250\) rounds with \(n_t=10\), sampling batches with replacement from a fixed response pool, with every method receiving the same ordered stream. The main setting uses \(c(z)=1\) and \(B_t=5\); Appendix~\ref{app:distribution-shift} considers nonstationary streams. For each model--dataset pair, we choose \(\varepsilon\) so the bound in Equation~\eqref{eq:online-type-I-bound} approximately matches \textsc{LLM-Select}'s online Type-I error. We report running-average Type-I and Type-II errors and editor cost; Appendix~\ref{app:operating-point-tradeoffs} gives complementary metrics and operating points. We audit \(200\) held-out interactions (\(50\) per dataset) with a human and GPT-5.6 Sol under the same confirm-versus-edit criterion; Appendix~\ref{app:human-audit} reports their agreement.

\subsection{Results}
\label{subsec:main-results}

\begin{figure*}[t!]
    \centering
    \includegraphics[width=\textwidth]{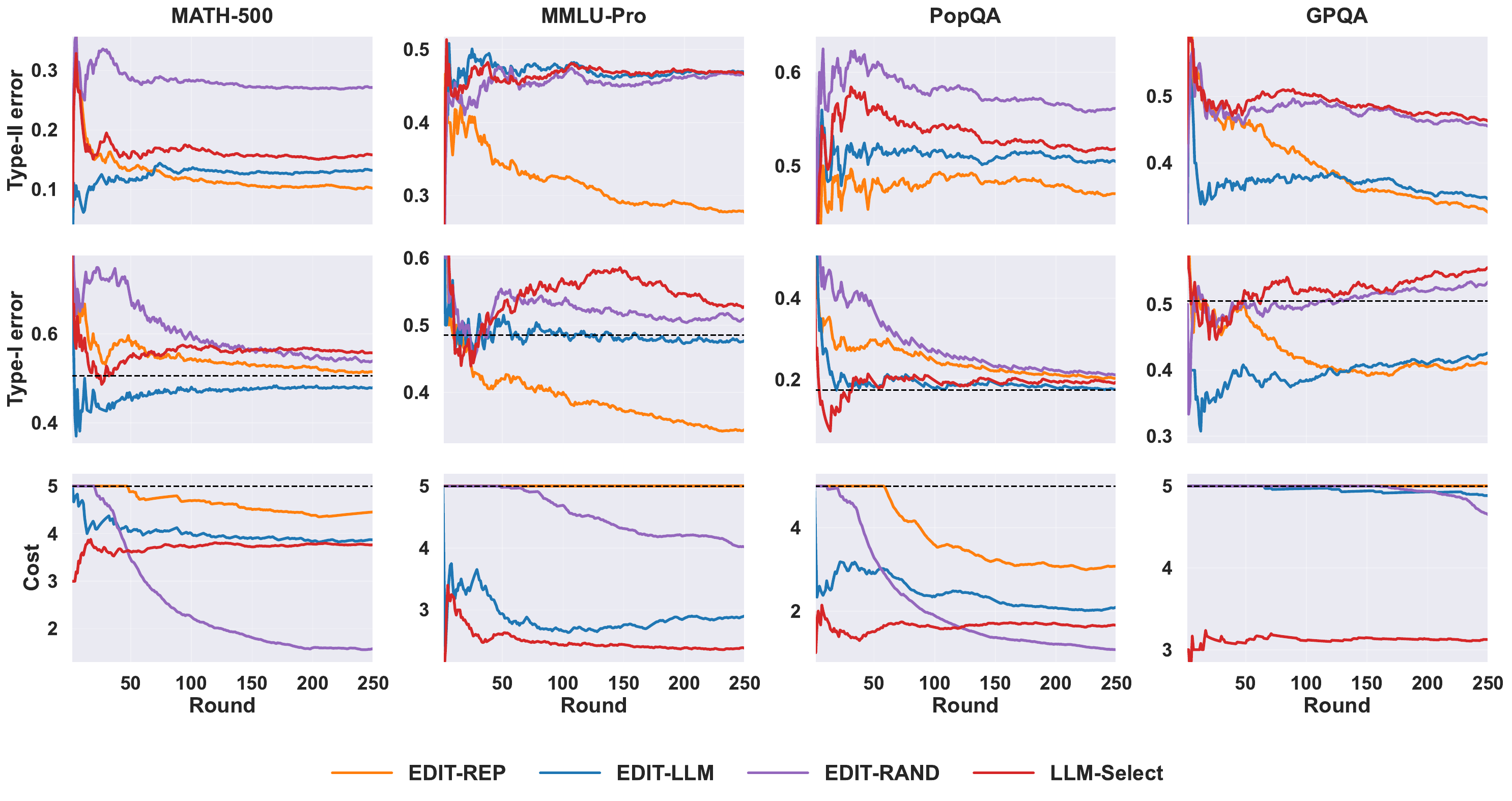}
    \caption{
    Main results with Llama-3.3-70B-Instruct. Columns correspond to MATH-500, MMLU-Pro, PopQA, and GPQA. Rows report running-average Type-II error, Type-I error, and editor cost; dashed horizontal lines mark the configured Type-I error tolerance and per-round editor budget.}
    \label{fig:main-llama-results}
\end{figure*}

Figure~\ref{fig:main-llama-results} shows the online trajectories for Llama-3.3-70B-Instruct. Across all four datasets, \EDITREP{} generally attains lower Type-II error than \LLMSelect{} with comparable or more stable Type-I behavior. This supports separating an edit score, adaptive Type-I control, and budget mechanism rather than direct LLM batch selection. Unlike \EDIT{}, \LLMSelect{} cannot explicitly adapt to the observed confirmation rate, making its Type-I behavior less consistent across datasets.

Score quality also strongly affects discovery. \EDITREP{} attains the lowest final Type-II error on all four datasets, with especially clear gains on MMLU-Pro and GPQA; \EDITLLM{} also improves over \EDITRAND{} on several datasets. Since all variants share the selection rule, the differences reflect score information: the controller sets selectivity and the score prioritizes interactions within budget.

The main \EDITREP{} score uses a 22.7M-parameter MiniLM encoder, tiny relative to the 70B response model. Appendix~\ref{app:embedding-size} shows Qwen3-0.6B, 1.7B, and 4B representations progressively improving edit discovery, with 4B attaining the lowest final Type-II error on all four datasets, suggesting stronger encoders can further improve performance in this setting.

The Type-I trajectories show \(\alpha_t\) adapting to the observed confirmation rate toward the configured tolerance. Cost must be interpreted jointly with both errors: lower cost is not preferable if a conservative policy misses more editable responses. We therefore compare Type-II error under the prescribed Type-I tolerance and budget, where \EDITREP{} is consistently strongest.

\section{Conclusion and Limitations}
\label{sec:conclusion}

We introduced edit-discovery optimization and \EDIT{}, a cost-aware online method that combines edit scoring with adaptive Type-I control and per-round budget enforcement. The population analysis shows that optimal selection follows a cost-adjusted threshold on edit probability, motivating the decomposition used by \EDIT{}. In the online setting, the method controls the confirmation rate among selected interactions without distributional assumptions while satisfying the editor budget by construction. Across four reasoning and factual question-answering benchmarks, \EDITREP{} reduces missed edits relative to direct LLM-based batch selection while maintaining reliable Type-I behavior, showing the value of separating edit scoring, online control, and budget allocation. The results further indicate that score quality matters: even a lightweight representation model provides useful rankings, while stronger encoders can further improve edit discovery. This modular design allows stronger edit predictors to be incorporated without changing the online control rule.

Our experiments use benchmark-derived edit labels, supported by a limited 200-example human audit, rather than large-scale human editing. We also evaluate edit discovery itself rather than the downstream value of incorporating the resulting corrections into model training or adaptation. Future work can study richer human editing workflows, heterogeneous editing costs, and whether discovered edits translate into measurable downstream model improvements in realistic human--AI deployment settings with diverse users, domains, editor expertise, and resource constraints.

\section*{Acknowledgments}
SK, GP, and HH acknowledge support from EnCORE, the Institute for Emerging CORE Methods in Data Science. SK also acknowledges support from an AWS gift to Penn Engineering’s ASSET Center for Trustworthy AI.

\bibliographystyle{plainnat}
\bibliography{references}

@inproceedings{hendrycks2021math,
  title     = {Measuring Mathematical Problem Solving With the {MATH} Dataset},
  author    = {Hendrycks, Dan and Burns, Collin and Kadavath, Saurav and Arora, Akul and Basart, Steven and Tang, Eric and Song, Dawn and Steinhardt, Jacob},
  booktitle = {Advances in Neural Information Processing Systems: Datasets and Benchmarks Track},
  year      = {2021}
}

@article{ramalingam2025relationship,
  title={The relationship between no-regret learning and online conformal prediction},
  author={Ramalingam, Ramya and Kiyani, Shayan and Roth, Aaron},
  journal={arXiv preprint arXiv:2502.10947},
  year={2025}
}

@article{gibbs2021adaptive,
  title={Adaptive conformal inference under distribution shift},
  author={Gibbs, Isaac and Candes, Emmanuel},
  journal={Advances in Neural Information Processing Systems},
  volume={34},
  pages={1660--1672},
  year={2021}
}

@article{noorani2026multi,
  title={Multi-round human-ai collaboration with user-specified requirements},
  author={Noorani, Sima and Kiyani, Shayan and Hassani, Hamed and Pappas, George},
  journal={arXiv preprint arXiv:2602.17646},
  year={2026}
}

@inproceedings{lightman2024verify,
  title     = {Let's Verify Step by Step},
  author    = {Lightman, Hunter and Kosaraju, Vineet and Burda, Yuri and Edwards, Harrison and Baker, Bowen and Lee, Teddy and Leike, Jan and Schulman, John and Sutskever, Ilya and Cobbe, Karl},
  booktitle = {International Conference on Learning Representations},
  year      = {2024}
}

@inproceedings{wang2024mmlupro,
  title     = {{MMLU-Pro}: A More Robust and Challenging Multi-Task Language Understanding Benchmark},
  author    = {Wang, Yubo and Ma, Xueguang and Zhang, Ge and Ni, Yuansheng and Chandra, Abhranil and Guo, Shiguang and Ren, Weiming and Arulraj, Aaran and He, Xuan and Jiang, Ziyan and Li, Tianle and Ku, Max and Wang, Kai and Zhuang, Alex and Fan, Rongqi and Yue, Xiang and Chen, Wenhu},
  booktitle = {Advances in Neural Information Processing Systems},
  volume    = {37},
  year      = {2024},
  doi       = {10.52202/079017-3018}
}

@inproceedings{mallen2023popqa,
  title     = {When Not to Trust Language Models: Investigating Effectiveness of Parametric and Non-Parametric Memories},
  author    = {Mallen, Alex and Asai, Akari and Zhong, Victor and Das, Rajarshi and Khashabi, Daniel and Hajishirzi, Hannaneh},
  booktitle = {Proceedings of the 61st Annual Meeting of the Association for Computational Linguistics},
  year      = {2023}
}

@inproceedings{rein2023gpqa,
  title     = {{GPQA}: A Graduate-Level Google-Proof Q\&A Benchmark},
  author    = {Rein, David and Hou, Betty Li and Stickland, Asa Cooper and Petty, Jackson and Pang, Richard Yuanzhe and Dirani, Julien and Michael, Julian and Bowman, Samuel R.},
  booktitle = {First Conference on Language Modeling},
  year      = {2024},
  url       = {https://openreview.net/forum?id=Ti67584b98}
}

@article{grattafiori2024llama3,
  title   = {The {Llama 3} Herd of Models},
  author  = {Grattafiori, Aaron and Dubey, Abhimanyu and Jauhri, Abhinav and Pandey, Abhinav and Kadian, Abhishek and Al-Dahle, Ahmad and others},
  journal = {arXiv preprint arXiv:2407.21783},
  year    = {2024}
}

@article{yang2025qwen3,
  title   = {{Qwen3} Technical Report},
  author  = {Yang, An and Li, Anfeng and Yang, Baosong and Zhang, Beichen and Hui, Binyuan and Zheng, Bo and Yu, Bowen and Gao, Chang and Huang, Chengen and Lv, Chenxu and others},
  journal = {arXiv preprint arXiv:2505.09388},
  year    = {2025}
}

@inproceedings{wang2020minilm,
  title     = {{MiniLM}: Deep Self-Attention Distillation for Task-Agnostic Compression of Pre-Trained Transformers},
  author    = {Wang, Wenhui and Wei, Furu and Dong, Li and Bao, Hangbo and Yang, Nan and Zhou, Ming},
  booktitle = {Advances in Neural Information Processing Systems},
  year      = {2020}
}

@inproceedings{reimers2019sbert,
  title     = {Sentence-{BERT}: Sentence Embeddings using Siamese {BERT}-Networks},
  author    = {Reimers, Nils and Gurevych, Iryna},
  booktitle = {Proceedings of the 2019 Conference on Empirical Methods in Natural Language Processing and the 9th International Joint Conference on Natural Language Processing},
  year      = {2019}
}

@techreport{settles2009active,
  title       = {Active Learning Literature Survey},
  author      = {Settles, Burr},
  institution = {University of Wisconsin--Madison},
  number      = {1648},
  year        = {2009}
}

@inproceedings{zhang2022active,
  title     = {A Survey of Active Learning for Natural Language Processing},
  author    = {Zhang, Zhisong and Strubell, Emma and Hovy, Eduard},
  booktitle = {Proceedings of the 2022 Conference on Empirical Methods in Natural Language Processing},
  year      = {2022}
}

@article{kadavath2022know,
  title   = {Language Models (Mostly) Know What They Know},
  author  = {Kadavath, Saurav and Conerly, Tom and Askell, Amanda and Henighan, Tom and Drain, Dawn and Perez, Ethan and Schiefer, Nicholas and Hatfield-Dodds, Zac and DasSarma, Nova and Tran-Johnson, Eli and others},
  journal = {arXiv preprint arXiv:2207.05221},
  year    = {2022}
}

@inproceedings{zheng2023judge,
  title     = {Judging {LLM}-as-a-Judge with {MT-Bench} and Chatbot Arena},
  author    = {Zheng, Lianmin and Chiang, Wei-Lin and Sheng, Ying and Zhuang, Siyuan and Wu, Zhanghao and Zhuang, Yonghao and Lin, Zi and Li, Zhuohan and Li, Dacheng and Xing, Eric P. and Zhang, Hao and Gonzalez, Joseph E. and Stoica, Ion},
  booktitle = {Advances in Neural Information Processing Systems},
  year      = {2023}
}

@article{angelopoulos2023ppi,
  title   = {Prediction-Powered Inference},
  author  = {Angelopoulos, Anastasios N. and Bates, Stephen and Fannjiang, Clara and Jordan, Michael I. and Zrnic, Tijana},
  journal = {Science},
  volume  = {382},
  number  = {6671},
  pages   = {669--674},
  year    = {2023},
  doi     = {10.1126/science.adi6000}
}

@inproceedings{fisch2024stratppi,
  title     = {Stratified Prediction-Powered Inference for Hybrid Language Model Evaluation},
  author    = {Fisch, Adam and Maynez, Joshua and Hofer, R. Alex and Dhingra, Bhuwan and Globerson, Amir and Cohen, William W.},
  booktitle = {Advances in Neural Information Processing Systems},
  volume    = {37},
  year      = {2024}
}

@inproceedings{zrnic2024active,
  title     = {Active Statistical Inference},
  author    = {Zrnic, Tijana and Cand\`es, Emmanuel},
  booktitle = {Proceedings of the 41st International Conference on Machine Learning},
  series    = {Proceedings of Machine Learning Research},
  volume    = {235},
  pages     = {62993--63010},
  year      = {2024}
}

@article{angelopoulos2025costoptimal,
  title   = {Cost-Optimal Active {AI} Model Evaluation},
  author  = {Angelopoulos, Anastasios N. and Eisenstein, Jacob and Berant, Jonathan and Agarwal, Alekh and Fisch, Adam},
  journal = {arXiv preprint arXiv:2506.07949},
  year    = {2025}
}

@inproceedings{candes2026pac,
  title     = {Probably Approximately Correct Labels},
  author    = {Cand\`es, Emmanuel J. and Ilyas, Andrew and Zrnic, Tijana},
  booktitle = {Proceedings of the 43rd International Conference on Machine Learning},
  year      = {2026}
}

@article{jin2023selection,
  title   = {Selection by Prediction with Conformal p-values},
  author  = {Jin, Ying and Cand\`es, Emmanuel J.},
  journal = {Journal of Machine Learning Research},
  volume  = {24},
  number  = {244},
  pages   = {1--41},
  year    = {2023}
}

@inproceedings{gui2024conformalalignment,
  title     = {Conformal Alignment: Knowing When to Trust Foundation Models with Guarantees},
  author    = {Gui, Yu and Jin, Ying and Ren, Zhimei},
  booktitle = {Advances in Neural Information Processing Systems},
  volume    = {37},
  year      = {2024}
}

@article{gui2025acs,
  title   = {{ACS}: An Interactive Framework for Conformal Selection},
  author  = {Gui, Yu and Jin, Ying and Nair, Yash and Ren, Zhimei},
  journal = {arXiv preprint arXiv:2507.15825},
  year    = {2025}
}

@article{liu2025ocsarc,
  title     = {Online Conformal Selection with Accept-to-Reject Changes},
  author    = {Liu, Kangdao and Xi, Huajun and Vong, Chi Man and Wei, Hongxin},
  journal   = {Proceedings of the AAAI Conference on Artificial Intelligence},
  volume    = {40},
  number    = {28},
  pages     = {23765--23773},
  year      = {2026},
  doi       = {10.1609/aaai.v40i28.39551}
}

@article{bai2026score,
  title   = {Conformal Selective Prediction with General Risk Control},
  author  = {Bai, Tian and Jin, Ying},
  journal = {arXiv preprint arXiv:2603.24704},
  year    = {2026}
}

@inproceedings{geifman2019selectivenet,
  title     = {{SelectiveNet}: A Deep Neural Network with an Integrated Reject Option},
  author    = {Geifman, Yonatan and El-Yaniv, Ran},
  booktitle = {Proceedings of the 36th International Conference on Machine Learning},
  series    = {Proceedings of Machine Learning Research},
  volume    = {97},
  pages     = {2151--2159},
  year      = {2019}
}

@inproceedings{madras2018defer,
  title     = {Predict Responsibly: Improving Fairness and Accuracy by Learning to Defer},
  author    = {Madras, David and Pitassi, Toniann and Zemel, Richard},
  booktitle = {Advances in Neural Information Processing Systems},
  volume    = {31},
  year      = {2018}
}

@inproceedings{mozannar2020defer,
  title     = {Consistent Estimators for Learning to Defer to an Expert},
  author    = {Mozannar, Hussein and Sontag, David},
  booktitle = {Proceedings of the 37th International Conference on Machine Learning},
  series    = {Proceedings of Machine Learning Research},
  volume    = {119},
  pages     = {7076--7087},
  year      = {2020}
}

@inproceedings{angelopoulos2024risk,
  title     = {Conformal Risk Control},
  author    = {Angelopoulos, Anastasios N. and Bates, Stephen and Fisch, Adam and Lei, Lihua and Schuster, Tal},
  booktitle = {International Conference on Learning Representations},
  year      = {2024},
  url       = {https://openreview.net/forum?id=33XGfHLtZg}
}

@article{liang2023helm,
  title   = {Holistic Evaluation of Language Models},
  author  = {Liang, Percy and Bommasani, Rishi and Lee, Tony and Tsipras, Dimitris and Soylu, Dilara and Yasunaga, Michihiro and Zhang, Yian and Narayanan, Deepak and Wu, Yuhuai and Kumar, Ananya and others},
  journal = {Transactions on Machine Learning Research},
  year    = {2023}
}

@misc{huang2026conformallabeling,
      title={Model-agnostic Selective Labeling with Provable Statistical Guarantees}, 
      author={Huipeng Huang and Wenbo Liao and Huajun Xi and Hao Zeng and Mengchen Zhao and Hongxin Wei},
      year={2026},
      eprint={2510.14581},
      archivePrefix={arXiv},
      primaryClass={cs.LG},
      url={https://arxiv.org/abs/2510.14581}, 
}

@article{noorani2025collaborative,
  title         = {Human-{AI} Collaborative Uncertainty Quantification},
  author        = {Noorani, Sima and Kiyani, Shayan and Pappas, George and Hassani, Hamed},
  journal       = {arXiv preprint arXiv:2510.23476},
  year          = {2025},
  eprint        = {2510.23476},
  archivePrefix = {arXiv},
  primaryClass  = {cs.AI},
  doi           = {10.48550/arXiv.2510.23476},
  url           = {https://arxiv.org/abs/2510.23476}
}

@article{kiyani2026support,
  title         = {Strategic Decision Support for {AI} Agents},
  author        = {Kiyani, Shayan and Noorani, Sima and Pappas, George and Hassani, Hamed},
  journal       = {arXiv preprint arXiv:2606.12587},
  year          = {2026},
  eprint        = {2606.12587},
  archivePrefix = {arXiv},
  primaryClass  = {cs.AI},
  doi           = {10.48550/arXiv.2606.12587},
  url           = {https://arxiv.org/abs/2606.12587}
}

@article{kiyani2026verification,
  title         = {When to Trust the Cheap Check: Weak and Strong Verification for Reasoning},
  author        = {Kiyani, Shayan and Noorani, Sima and Pappas, George and Hassani, Hamed},
  journal       = {arXiv preprint arXiv:2602.17633},
  year          = {2026},
  eprint        = {2602.17633},
  archivePrefix = {arXiv},
  primaryClass  = {cs.LG},
  doi           = {10.48550/arXiv.2602.17633},
  url           = {https://arxiv.org/abs/2602.17633}
}

@inproceedings{straitouri2024counterfactual,
  title     = {Designing Decision Support Systems using Counterfactual Prediction Sets},
  author    = {Straitouri, Eleni and Gomez Rodriguez, Manuel},
  booktitle = {Proceedings of the 41st International Conference on Machine Learning},
  series    = {Proceedings of Machine Learning Research},
  volume    = {235},
  pages     = {46722--46744},
  year      = {2024},
  publisher = {PMLR},
  url       = {https://proceedings.mlr.press/v235/straitouri24a.html},
  eprint    = {2306.03928},
  archivePrefix = {arXiv}
}

@inproceedings{straitouri2023expert,
  title     = {Improving Expert Predictions with Conformal Prediction},
  author    = {Straitouri, Eleni and Wang, Lequn and Okati, Nastaran and Gomez Rodriguez, Manuel},
  booktitle = {Proceedings of the 40th International Conference on Machine Learning},
  series    = {Proceedings of Machine Learning Research},
  volume    = {202},
  pages     = {32633--32653},
  year      = {2023},
  publisher = {PMLR},
  url       = {https://proceedings.mlr.press/v202/straitouri23a.html}
}

\newpage
\appendix
\startcontents[appendix]

\clearpage
\section*{Appendix Contents}
\printcontents[appendix]{}{1}{\setcounter{tocdepth}{2}}

\newpage

\section{Extended Related Work}\label{sec:related-work}

\paragraph{Efficient AI evaluation.}
A growing line of work studies how model predictions can reduce the cost of evaluation. Prediction-powered inference combines a small labeled sample with predictions on a much larger unlabeled sample to estimate population quantities with valid uncertainty \citep{angelopoulos2023ppi}. This idea has been specialized to hybrid language-model evaluation through stratified sampling \citep{fisch2024stratppi}, while cost-optimal active evaluation allocates queries between weak and strong raters to obtain an efficient estimate of mean strong-rater performance \citep{angelopoulos2025costoptimal}. Active statistical inference similarly chooses which labels to acquire in order to improve the precision of confidence intervals or tests \citep{zrnic2024active}. These works are closely related to our motivation because they explicitly account for expensive human evaluation. Their objective, however, is accurate statistical inference about an aggregate quantity. In contrast, our goal is to identify the individual interactions that should be edited, while controlling confirmations among selected interactions and respecting the available editor budget. Standard LLM evaluation frameworks and LLM-based judges also focus on measuring model quality \citep{liang2023helm,zheng2023judge}, rather than deciding which interactions should enter an online human feedback loop.

\paragraph{AI-assisted data collection.}
Classical active learning selects examples whose labels are expected to improve a predictive model \citep{settles2009active,zhang2022active}. More recent work considers hybrid annotation pipelines in which an AI system supplies some labels and an expert supplies the remainder. PAC labeling uses model uncertainty to construct a completed dataset with bounded aggregate labeling error \citep{candes2026pac}, whereas Conformal Labeling controls the false discovery rate among labels assigned by the AI system \citep{huang2026conformallabeling}. Although these methods also reduce human annotation effort, their output is a fully labeled dataset in which skipped expert labels are replaced by model labels. In our setting, skipped responses remain unlabeled, and the purpose of selection is to discover corrections rather than to approximate all missing labels. This distinction is important in practice because our policy must learn from only the edit decisions revealed for the interactions it previously selected.

\paragraph{Conformal selection and online calibration.}
Conformal selection wraps a prediction model with calibrated $p$-values and selects units whose unobserved outcomes exceed a target while controlling false discoveries \citep{jin2023selection}. Conformal Alignment applies this perspective to foundation-model outputs and certifies a selected set as trustworthy \citep{gui2024conformalalignment}. Related extensions support interactive data analysis, irreversible online decisions, and general bounded risks \citep{gui2025acs,liu2025ocsarc,bai2026score}. These methods select outputs to accept or trust, whereas our positive decision sends an output to a human because it is likely to require correction. Reversing an acceptance rule does not preserve its guarantee or yield control of confirmations among selected items. Moreover, these methods generally require labeled calibration data or exchangeability-type conditions. Our threshold update is technically related to adaptive conformal inference and conformal risk control \citep{gibbs2021adaptive,angelopoulos2024risk}, but our guarantee applies to an arbitrary stream with labels revealed only after editor feedback and is coupled to a per-round cost constraint.

\paragraph{Selective prediction and learning to defer.}
Selective prediction allows a model to abstain on inputs for which its prediction is unreliable, including end-to-end methods that optimize the risk--coverage trade-off \citep{geifman2019selectivenet}. Learning to defer extends this formulation by explicitly accounting for the accuracy, bias, or cost of a downstream human expert \citep{madras2018defer,mozannar2020defer}. In these formulations, deferral is an alternative action used to improve the accuracy or utility of the deployed prediction system. Our setting differs in two ways. First, the model response has already been produced, and the human is asked to edit or confirm it in order to create corrective data. Second, a selection policy can fail both by selecting responses that are ultimately confirmed and by skipping responses that would have been edited. Standard deferral methods do not address this edit-discovery objective or the partial feedback induced by selective editor invocation.

\paragraph{Human--AI collaborative procedures.}
A complementary line of work designs procedures that combine human and AI judgments or selectively acquire costly assistance. \citep{straitouri2023expert} construct conformal prediction sets from which experts choose a label, balancing the risk of excluding the correct answer against the difficulty of choosing among many candidates. \citep{straitouri2024counterfactual} optimize this assistance through online learning, exploiting nested prediction sets and a counterfactual monotonicity assumption without requiring a parametric expert model. \citep{noorani2025collaborative} instead refine human-proposed prediction sets through AI-guided pruning and augmentation, with explicit counterfactual-harm and complementarity criteria and offline and online calibration algorithms. For reasoning systems, \citep{kiyani2026verification} study policies that accept, reject, or request strong verification using a weak score, and develop randomized online procedures that control incorrect acceptance and rejection. \citep{kiyani2026support} minimize support-call frequency subject to controlling missed beneficial support, using adaptive thresholding and randomized exploration under partial feedback. These works provide algorithmic foundations for human--AI collaboration and selective assistance. Our formulation focuses on discovering editor modifications under constraints on confirmations and interaction-specific review costs.

\section{Proofs}
\label{app:proofs}

\subsection{Proof of Theorem \ref{thm:optimal-population-policy}}
\label{app:proof-optimal-population-policy}
\begin{proof}
First, we rewrite the two conditional probabilities in terms of \(\eta\). For any policy \(\pi\), by
the law of total expectation,
\[
\Prob{A=1,S_\pi=0}
=
\E[\Prob{A=1,S_\pi=0\mid Z}]
=
\E[(1-\pi(Z))\eta(Z)].
\]
Also,
\(
    \Prob{S_\pi=0}=\E[1-\pi(Z)]
\).
Therefore,
\[
\Prob{A=1\mid S_\pi=0}
=
\frac{\Prob{A=1,S_\pi=0}}{\Prob{S_\pi=0}}
=
\frac{\E[(1-\pi(Z))\eta(Z)]}{\E[1-\pi(Z)]}.
\]
Similarly,
\[
\Prob{A=0,S_\pi=1}
=
\E[\Prob{A=0,S_\pi=1\mid Z}]
=
\E[\pi(Z)(1-\eta(Z))],
\]
and
\(
    \Prob{S_\pi=1}=\E[\pi(Z)]
\).
Hence,
\[
\Prob{A=0\mid S_\pi=1}
=
\frac{\Prob{A=0,S_\pi=1}}{\Prob{S_\pi=1}}
=
\frac{\E[\pi(Z)(1-\eta(Z))]}{\E[\pi(Z)]}.
\]
Thus the constraint \(\Prob{A=0\mid S_\pi=1}\le \varepsilon\) is equivalent to
\[
    \E[\pi(Z)(1-\eta(Z)-\varepsilon)]\le 0.
\]
Indeed, if \(\E[\pi(Z)]>0\), this follows by multiplying both sides by \(\E[\pi(Z)]\). If
\(\E[\pi(Z)]=0\), both sides reduce to \(0\le0\) under the population convention \(0/0=0\).

If \(\Prob{S_{\pi^\star}=0}=0\), then \(\pi^\star(z)=1\) almost surely. In this case the claimed
threshold form follows by taking \(\alpha^\star=0\) and \(\beta^\star=0\), with the boundary chosen
so that the policy selects almost surely.

It remains to consider the case \(\Prob{S_{\pi^\star}=0}>0\). Let \(\pi^\star\) be an optimal policy.
By optimality, for every feasible policy \(\pi\) with \(\E[1-\pi(Z)]>0\),
\[
\Prob{A=1\mid S_\pi=0}
\ge
\Prob{A=1\mid S_{\pi^\star}=0}.
\]
Equivalently,
\[
\frac{\E[(1-\pi(Z))\eta(Z)]}{\E[1-\pi(Z)]}
\ge
\Prob{A=1\mid S_{\pi^\star}=0}.
\]
Since \(\E[1-\pi(Z)]>0\), this is equivalent to
\[
\E\left[
(1-\pi(Z))
\left(
\eta(Z)-\Prob{A=1\mid S_{\pi^\star}=0}
\right)
\right]\ge 0.
\]
For \(\pi=\pi^\star\), equality holds:
\[
\E\left[
(1-\pi^\star(Z))
\left(
\eta(Z)-\Prob{A=1\mid S_{\pi^\star}=0}
\right)
\right]=0.
\]
Therefore, \(\pi^\star\) solves the auxiliary linear problem
\begin{equation}
\label{eq:auxiliary-linear-problem}
\begin{aligned}
\min_{\pi:\cZ\to[0,1]}
\quad &
\E\left[
(1-\pi(Z))
\left(
\eta(Z)-\Prob{A=1\mid S_{\pi^\star}=0}
\right)
\right]
\\
\text{s.t.}\quad
&
\E[\pi(Z)(1-\eta(Z)-\varepsilon)]\le 0,
\\
&
\E[\pi(Z)c(Z)]\le B .
\end{aligned}
\end{equation}
The policy space in \eqref{eq:edo} already allows \(0\le \pi\le1\), so no relaxation is introduced here.
Since the objective and constraints in \eqref{eq:auxiliary-linear-problem} are linear in \(\pi\), the
pointwise optimizer is binary except possibly on equality boundaries.

We next verify Slater's condition, after handling two degenerate cases directly.
First, if \(B=0\), then \(c(Z)>0\) implies that every feasible policy satisfies
\(\pi(Z)=0\) almost surely. The claimed threshold form therefore follows by taking
\(\alpha^\star=1\) and \(\beta^\star=0\), with the boundary chosen to remain unselected.

Next, suppose
\(
    \Prob{\eta(Z)>1-\varepsilon}=0
\).
Then \(1-\eta(Z)-\varepsilon\ge0\) almost surely, and feasibility implies
\[
    \pi^\star(Z)(1-\eta(Z)-\varepsilon)=0
    \qquad\text{almost surely}.
\]
Hence \(\pi^\star\) can select only on the set
\(\{\eta(Z)=1-\varepsilon\}\). The claimed threshold form follows by taking
\(\alpha^\star=1-\varepsilon\) and \(\beta^\star=0\), with the boundary chosen to reproduce
\(\pi^\star\).

We may therefore assume that \(B>0\) and
\(\Prob{\eta(Z)>1-\varepsilon}>0\). Since \(c(Z)<\infty\), there exists some
\(M<\infty\) such that the measurable set
\[
    H_M
    :=
    \{z:\eta(z)>1-\varepsilon,\ c(z)\le M\}
\]
has positive probability. Let
\[
    \bar\pi(z)=\delta\,\1\{z\in H_M\},
\]
where \(\delta\in(0,1]\) is chosen sufficiently small that
\(
    \delta M\Prob{Z\in H_M}<B
\).
Then
\[
\E[\bar\pi(Z)(1-\eta(Z)-\varepsilon)]<0
\qquad \text{and} \qquad
\E[\bar\pi(Z)c(Z)]
\le
\delta M\Prob{Z\in H_M}
<
B.
\]
Thus the problem in \eqref{eq:auxiliary-linear-problem} satisfies Slater's condition. Viewing
\(\pi\) as an element of \(L^\infty(P_Z)\), its objective and constraints are continuous linear
functionals, and standard convex duality therefore gives strong duality; the KKT conditions are
necessary and sufficient for optimality.

For multipliers \(\lambda_1,\lambda_2\ge0\), define the Lagrangian
\[
\begin{aligned}
\mathcal L(\pi,\lambda_1,\lambda_2)
&=
\E\left[
(1-\pi(Z))
\left(
\eta(Z)-\Prob{A=1\mid S_{\pi^\star}=0}
\right)
\right]
\\
&\qquad
+
\lambda_1\E[\pi(Z)(1-\eta(Z)-\varepsilon)]
+
\lambda_2\bigl(\E[\pi(Z)c(Z)]-B\bigr)
\\
&=
\E\left[
\eta(Z)-\Prob{A=1\mid S_{\pi^\star}=0}
\right]
-\lambda_2 B
\\
&\qquad
+
\E\Bigl[
\pi(Z)
\bigl(
\Prob{A=1\mid S_{\pi^\star}=0}
-\eta(Z)
+
\lambda_1(1-\eta(Z)-\varepsilon)
+
\lambda_2 c(Z)
\bigr)
\Bigr].
\end{aligned}
\]
By the KKT conditions, there exist optimal dual variables
\(\lambda_1^\star,\lambda_2^\star\ge0\) such that:
\begin{itemize}
    \item
    \(
    \E[\pi^\star(Z)(1-\eta(Z)-\varepsilon)]\le 0,
    \qquad
    \E[\pi^\star(Z)c(Z)]-B\le 0.
    \)

    \item
    \(
    \lambda_1^\star,\lambda_2^\star\ge 0.
    \)

    \item
    \(
    \lambda_1^\star
    \E[\pi^\star(Z)(1-\eta(Z)-\varepsilon)]
    =0,
    \qquad
    \lambda_2^\star
    \bigl(\E[\pi^\star(Z)c(Z)]-B\bigr)=0.
    \)

    \item
    \(
    \pi^\star\in
    \argmin_{0\le \pi\le 1}
    \mathcal L(\pi,\lambda_1^\star,\lambda_2^\star).
    \)
\end{itemize}

Now fix \(\lambda_1=\lambda_1^\star\) and \(\lambda_2=\lambda_2^\star\). Since the Lagrangian separates pointwise,
minimizing \(\mathcal L(\pi,\lambda_1^\star,\lambda_2^\star)\) over \(0\le \pi\le1\) is the same as solving,
for almost every \(z\),
\[
\min_{0\le t\le 1}
t
\bigl(
\Prob{A=1\mid S_{\pi^\star}=0}
-\eta(z)
+
\lambda_1^\star(1-\eta(z)-\varepsilon)
+
\lambda_2^\star c(z)
\bigr).
\]
Thus the minimizer is
\[
t^\star(z)=
\begin{cases}
1,
& \text{if }
\Prob{A=1\mid S_{\pi^\star}=0}
-\eta(z)
+
\lambda_1^\star(1-\eta(z)-\varepsilon)
+
\lambda_2^\star c(z)<0,\\[4pt]
0,
& \text{if }
\Prob{A=1\mid S_{\pi^\star}=0}
-\eta(z)
+
\lambda_1^\star(1-\eta(z)-\varepsilon)
+
\lambda_2^\star c(z)>0,\\[4pt]
\text{any value in }[0,1],
& \text{if }
\Prob{A=1\mid S_{\pi^\star}=0}
-\eta(z)
+
\lambda_1^\star(1-\eta(z)-\varepsilon)
+
\lambda_2^\star c(z)=0.
\end{cases}
\]
The last case is precisely the threshold boundary and can be resolved by randomization if needed.

Finally, since \(\lambda_1^\star,\lambda_2^\star\ge0\),
\[
\begin{aligned}
&
\Prob{A=1\mid S_{\pi^\star}=0}
-\eta(z)
+
\lambda_1^\star(1-\eta(z)-\varepsilon)
+
\lambda_2^\star c(z)<0
\\
&\iff
(1+\lambda_1^\star)\eta(z)
>
\Prob{A=1\mid S_{\pi^\star}=0}
+\lambda_1^\star(1-\varepsilon)+\lambda_2^\star c(z)
\\
&\iff
\eta(z)
>
\frac{\Prob{A=1\mid S_{\pi^\star}=0}+\lambda_1^\star(1-\varepsilon)}
{1+\lambda_1^\star}
+
\frac{\lambda_2^\star}{1+\lambda_1^\star}c(z).
\end{aligned}
\]
Define
\[
\alpha^\star
:=
\frac{\Prob{A=1\mid S_{\pi^\star}=0}+\lambda_1^\star(1-\varepsilon)}
{1+\lambda_1^\star},
\qquad
\beta^\star
:=
\frac{\lambda_2^\star}{1+\lambda_1^\star}.
\]
Since \(0\le \varepsilon<1\), \(\Prob{A=1\mid S_{\pi^\star}=0}\ge0\), and
\(\lambda_1^\star,\lambda_2^\star\ge0\), we have \(\alpha^\star,\beta^\star\ge0\). Therefore,
\[
\pi^\star(z)
=
\1\left\{
\eta(z)
>
\alpha^\star+\beta^\star c(z)
\right\},
\]
with possible randomization on the boundary
\(\eta(z)=\alpha^\star+\beta^\star c(z)\). This proves the claimed form of the optimal policy.
\end{proof}

\subsection{Proof of Theorem \ref{thm:online-binary-constraint-guarantee}}
\label{app:proof-online-binary-constraint-guarantee}
\begin{proof}
The budget constraint holds by the definition of \(\beta_t\). Indeed, from
\eqref{eq:online-alpha-beta-definition} and the \EDIT{} threshold rule
\eqref{eq:online-binary-threshold-rule},
\[
\sumi
\piEDITti\costti
=
\sumi
\1\left\{
\scoreti
>
\alpha_t+\beta_t\costti
\right\}
\costti
\le B_t.
\]

It remains to prove the average per-round Type-I constraint guarantee. Since \(\gamma_t=\gamma\), the update
rule gives
\[
\alpha_{t+1}
=
\max\left\{
0,
\alpha_t+
\gamma
\left(
\frac{
\sumi
\piEDITti(1-\Ati)
}{
1\lor
\sumi
\piEDITti
}
-\varepsilon
\right)
\right\}.
\]
Therefore,
\[
\alpha_{t+1}
\ge
\alpha_t+
\gamma
\left(
\frac{
\sumi
\piEDITti(1-\Ati)
}{
1\lor
\sumi
\piEDITti
}
-\varepsilon
\right) \implies
\frac{
\sumi
\piEDITti(1-\Ati)
}{
1\lor
\sumi
\piEDITti
}
-\varepsilon
\le
\frac{\alpha_{t+1}-\alpha_t}{\gamma}.
\]
Summing over \(t=1,\dots,T\), and using \(\alpha_1=0\), gives
\[
\begin{aligned}
\sumt
\left(
\frac{
\sumi
\piEDITti(1-\Ati)
}{
1\lor
\sumi
\piEDITti
}
-\varepsilon
\right)
&\le
\sumt
\frac{\alpha_{t+1}-\alpha_t}{\gamma}
=
\frac{\alpha_{T+1}-\alpha_1}{\gamma}
=
\frac{\alpha_{T+1}}{\gamma}.
\end{aligned}
\]

We now show that \(\alpha_{T+1}\le1+\gamma\). To do so, we prove the stronger claim that
\(
\alpha_t\le1+\gamma
\)
for all \(t=1,\dots,T+1\).
The proof is by induction on \(t\). Since \(\alpha_1=0\), the claim holds initially. If
\(\alpha_t\le1\), then the selected confirmation ratio is at most one, and therefore
\[
\alpha_{t+1}
\le
\alpha_t+\gamma
\le
1+\gamma.
\]
If \(\alpha_t>1\), then \(\scoreti\in[0,1]\), \(\beta_t\ge0\), and \(\costti>0\) imply that no item
can satisfy
\(
    \scoreti>\alpha_t+\beta_t\costti
\).
Hence \(\piEDITti=0\) for all \(i\). Therefore,
\[
\frac{
\sumi
\piEDITti(1-\Ati)
}{
1\lor
\sumi
\piEDITti
}
=0
\implies
\alpha_{t+1}
=
\max\left\{
0,
\alpha_t-\gamma\varepsilon
\right\}
\le
\alpha_t.
\]
Thus \(\alpha_t\le1+\gamma\) for all \(t\), and in particular
\(
    \alpha_{T+1}\le1+\gamma
\).
Consequently,
\[
\hErrI(\piEDIT)-\varepsilon
\le
\frac{1+\gamma}{\gamma T}.
\]
This proves the theorem.
\end{proof}

\section{\EDIT{} Pseudocode}\label{app:algorithm}
\begingroup
\renewcommand{\baselinestretch}{1.3}\selectfont

\begin{algorithm}[H]
\caption{\EDIT: \underline{E}dit \underline{D}iscovery via \underline{I}terative \underline{T}hresholding}
\label{alg:edit}
\KwRequire{Batches \(\{\cD_t\}_{t=1}^T\), budgets \(\{B_t\}_{t=1}^T\), Type-I error tolerance \(\varepsilon\), threshold stepsizes \(\{\gamma_t\}_{t=1}^T\), score \(\widetilde\eta_{\theta}\), score stepsizes \(\{\xi_t\}_{t=1}^T\), initial parameter \(\theta_1\)}
\(\alpha_1\gets0\)\;
\For{\(t=1,\dots,T\)}{
    compute \(\widetilde\eta_{\theta_t}(\Zti)\), \(\quad i=1,\dots,\nt\)\;
    \(\displaystyle
    \beta_t\gets
    \inf\left\{
    \beta\ge0:
    \sumi
    \costti
    \1\{\widetilde\eta_{\theta_t}(\Zti)>\alpha_t+\beta\costti\}
    \le B_t
    \right\}\)\;
    \(\piEDITti\gets
    \1\{\widetilde\eta_{\theta_t}(\Zti)>\alpha_t+\beta_t\costti\}\),
    \(\quad i=1,\dots,\nt\)\;
    observe \(\Ati\) for all \(i\) with \(\piEDITti=1\)\;
    \(\displaystyle
    \alpha_{t+1}\gets
    \max\left\{
    0,
    \alpha_t+
    \gamma_t
    \left(
    \frac{\sumi \piEDITti(1-\Ati)}
         {1\lor\sumi \piEDITti}
    -
    \varepsilon
    \right)
    \right\}\)\;
    \(\displaystyle
    \theta_{t+1}\gets
    \theta_t
    -
    \xi_t
    \frac{
    \sumi
    \piEDITti
    \left[
    \frac{\widetilde\eta_{\theta_t}(\Zti)-\Ati}
    {\widetilde\eta_{\theta_t}(\Zti)\left(1-\widetilde\eta_{\theta_t}(\Zti)\right)}
    \right]
    \nabla_{\theta_t} \widetilde\eta_{\theta_t}(\Zti)
    }{
    1\lor
    \sumi
    \piEDITti
    }\)\;
}
\end{algorithm}
\endgroup

\section{Additional Experiments}
\label{app:extended-experiments}

\subsection{Additional Tables and Figures}
\label{app:additional-tables}

Table~\ref{tab:epsilon-values} reports the Type-I error tolerances used for each
model--dataset pair, and Table~\ref{tab:model-accuracies} reports the corresponding
base-model accuracies on the benchmark interactions. Table~\ref{tab:common-hyperparameters}
summarizes the hyperparameters shared across the main online experiments.

\begin{table*}[t]
\centering

\begin{minipage}[t]{0.48\textwidth}
\vspace{0pt}
\centering
\scriptsize

\begin{minipage}[t][4.8\baselineskip][t]{\linewidth}
\captionof{table}{
Type-I error tolerance \(\varepsilon\) for each model--dataset pair.
The same Type-I error tolerance is used across all \EDIT{} score variants.
}
\label{tab:epsilon-values}
\end{minipage}

\renewcommand{\arraystretch}{1.18}
\setlength{\tabcolsep}{2pt}

\begin{tabularx}{\linewidth}{
    >{\hsize=0.82\hsize\raggedright\arraybackslash}X
    >{\hsize=1.20\hsize\centering\arraybackslash}X
    >{\hsize=0.98\hsize\centering\arraybackslash}X
}
\toprule
\textbf{Dataset}
&
\mbox{\textbf{Llama-3.3-70B}}
&
\mbox{\textbf{Qwen3-32B}}
\\
\midrule
\rowcolor{mathsoft}
\mbox{MATH-500} & 0.505 & 0.615 \\

\rowcolor{mmlusoft}
\mbox{MMLU-Pro} & 0.485 & 0.425 \\

\rowcolor{popsoft}
\mbox{PopQA} & 0.175 & 0.000 \\

\rowcolor{gpqasoft}
\mbox{GPQA} & 0.505 & 0.375 \\
\bottomrule
\end{tabularx}

\end{minipage}%
\hfill%
\begin{minipage}[t]{0.48\textwidth}
\vspace{0pt}
\centering
\scriptsize

\begin{minipage}[t][4.8\baselineskip][t]{\linewidth}
\captionof{table}{
Accuracy of the response models on the benchmark interactions.
}
\label{tab:model-accuracies}
\end{minipage}

\renewcommand{\arraystretch}{1.18}
\setlength{\tabcolsep}{2pt}

\begin{tabularx}{\linewidth}{
    >{\hsize=0.82\hsize\raggedright\arraybackslash}X
    >{\hsize=1.20\hsize\centering\arraybackslash}X
    >{\hsize=0.98\hsize\centering\arraybackslash}X
}
\toprule
\textbf{Dataset}
&
\mbox{\textbf{Llama-3.3-70B}}
&
\mbox{\textbf{Qwen3-32B}}
\\
\midrule
\rowcolor{mathsoft}
\mbox{MATH-500} & 0.730 & 0.794 \\

\rowcolor{mmlusoft}
\mbox{MMLU-Pro} & 0.526 & 0.558 \\

\rowcolor{popsoft}
\mbox{PopQA} & 0.444 & 0.220 \\

\rowcolor{gpqasoft}
\mbox{GPQA} & 0.540 & 0.458 \\
\bottomrule
\end{tabularx}

\end{minipage}

\end{table*}

\begin{table*}[t]
\centering
\caption{
Hyperparameters shared across the main online experiments.
}
\label{tab:common-hyperparameters}

\renewcommand{\arraystretch}{1.25}
\setlength{\tabcolsep}{7pt}

\begin{tabular*}{\textwidth}{
    @{\extracolsep{\fill}}
    ccccccccc
    @{}
}
\toprule
\(\boldsymbol{T}\)
&
\(\boldsymbol{n_t}\)
&
\(\boldsymbol{B_t}\)
&
\(\boldsymbol{c(z)}\)
&
\(\boldsymbol{\gamma}\)
&
\(\boldsymbol{\xi}\)
&
\(\boldsymbol{\alpha_1}\)
&
\(\boldsymbol{\theta_0}\)
&
\(\boldsymbol{d}\)
\\
\midrule
250
&
10
&
5
&
1
&
0.1
&
1.0
&
0
&
\(0.1 \times \boldsymbol{1}\)
&
384
\\
\bottomrule
\end{tabular*}

\end{table*}

\subsection{Human Audit of Benchmark-Derived Edit Labels}
\label{app:human-audit}

The benchmark experiments use task correctness as a scalable proxy for whether a response requires
modification, so we directly validate this assumption on \(200\) held-out prompt--response pairs:
\(50\) examples each from MATH-500, MMLU-Pro, PopQA, and GPQA. A human annotator evaluates every
example with the interface in Figure~\ref{fig:human-audit-ui} and assigns either \emph{confirm as-is}
or \emph{needs edit}; the interface exposes the user prompt, candidate response, and reference answer
and optionally records a reason category. GPT-5.6 Sol independently annotates the same \(200\)
examples using the same binary criterion and the common evaluation prompt reported in
Appendix~\ref{app:gpt-audit-prompt}. We treat the human labels as the primary validation target and
use GPT-5.6 Sol only as an additional consistency check.

\begin{figure*}[t]
    \centering
    \includegraphics[width=0.8\textwidth]{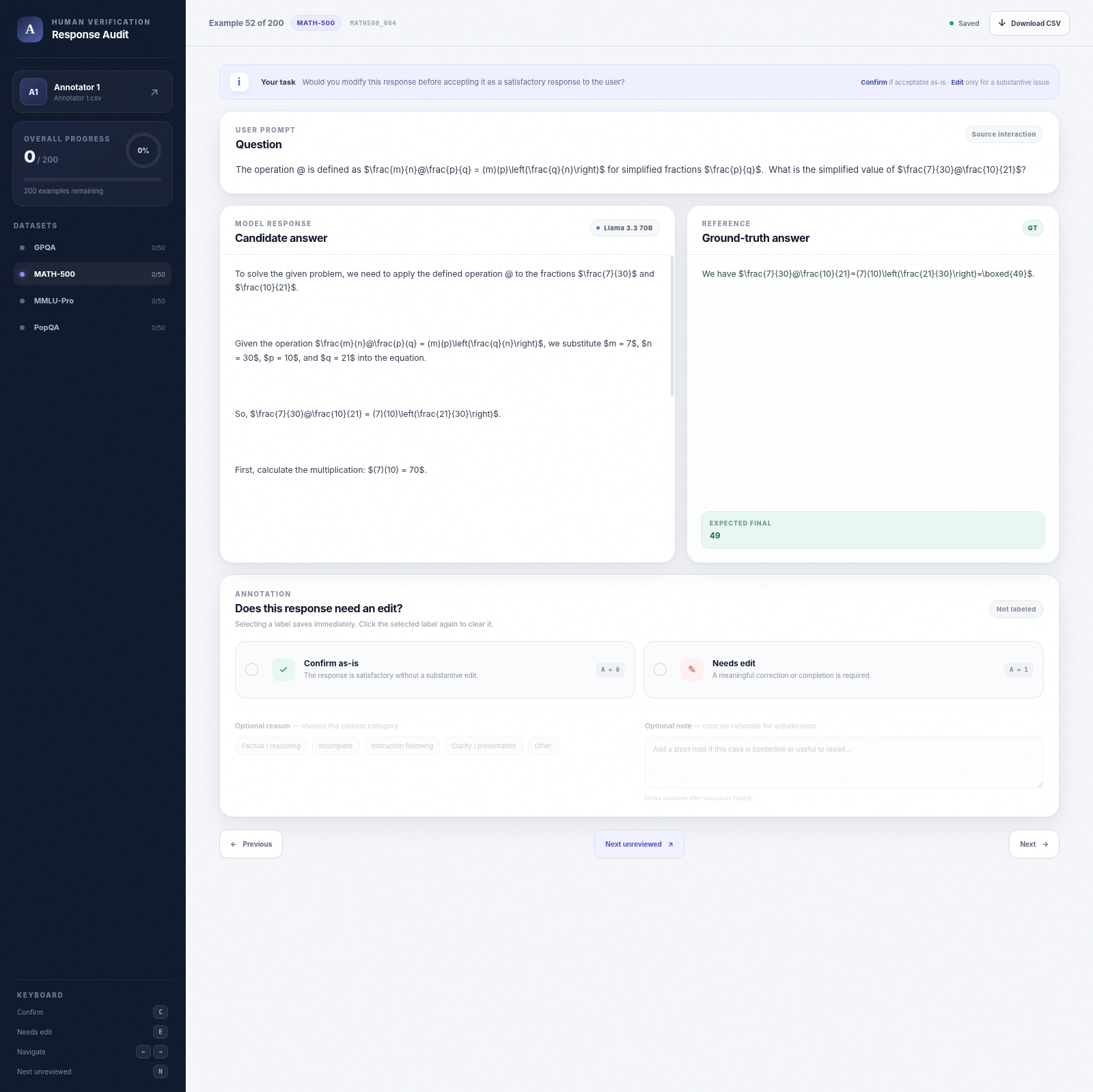}
    \caption{
    Human-audit interface. The human annotator inspects the user prompt, candidate response, and
    reference answer, then decides whether the response can be confirmed as-is or requires a
    substantive edit. Optional reason categories support later inspection of disagreement cases.
    }
    \label{fig:human-audit-ui}
\end{figure*}

The benchmark-derived labels closely track the human judgments at both the aggregate and
example levels. Across all \(200\) examples, benchmark and human labels agree on \(96.5\%\) of
interactions, with Cohen's \(\kappa=0.930\) and edit-class \(F_1=0.963\); their overall edit
rates are also nearly identical, \(0.470\) for the benchmark proxy versus \(0.465\) for the human
annotator. Dataset-level benchmark--human agreement remains between \(0.960\) and \(0.980\), with
\(\kappa>0.916\) for every dataset. GPT-5.6 Sol produces an edit rate of \(0.460\), agrees with the
benchmark in \(0.980\) of cases (\(\kappa=0.960\)), and agrees with the human annotator in \(0.975\)
of cases (\(\kappa=0.950\)). Table~\ref{tab:human-audit} gives the complete dataset-level breakdown.

\begin{table*}[t]
\centering
\caption{
Validation of the benchmark-derived edit labels. The human annotator is the primary validation target;
GPT-5.6 Sol is an independent model annotator. ``B--H agree'', \(\kappa_{BH}\), and \(F_{1,BH}\)
compare the benchmark with the human labels. ``H--GPT agree'' reports agreement between the human
annotator and GPT-5.6 Sol. The positive class is ``needs edit''.
}
\label{tab:human-audit}
\renewcommand{\arraystretch}{1.14}
\setlength{\tabcolsep}{5pt}
\scriptsize
\resizebox{\textwidth}{!}{%
\begin{tabular}{lcccccccc}
\toprule
\textbf{Dataset} &
\textbf{\(N\)} &
\textbf{Benchmark edit} &
\textbf{Human edit} &
\textbf{GPT-5.6 Sol edit} &
\textbf{B--H agree} &
\textbf{\(\kappa_{BH}\)} &
\textbf{\(F_{1,BH}\)} &
\textbf{H--GPT agree}
\\
\midrule
\cellcolor{mathsoft}MATH-500 & \cellcolor{mathsoft}50 & \cellcolor{mathsoft}0.280 & \cellcolor{mathsoft}0.300 & \cellcolor{mathsoft}0.260 & \cellcolor{mathsoft}0.980 & \cellcolor{mathsoft}0.951 & \cellcolor{mathsoft}0.966 & \cellcolor{mathsoft}0.960 \\
\cellcolor{mmlusoft}MMLU-Pro & \cellcolor{mmlusoft}50 & \cellcolor{mmlusoft}0.460 & \cellcolor{mmlusoft}0.420 & \cellcolor{mmlusoft}0.440 & \cellcolor{mmlusoft}0.960 & \cellcolor{mmlusoft}0.919 & \cellcolor{mmlusoft}0.955 & \cellcolor{mmlusoft}0.980 \\
\cellcolor{popsoft}PopQA & \cellcolor{popsoft}50 & \cellcolor{popsoft}0.560 & \cellcolor{popsoft}0.520 & \cellcolor{popsoft}0.540 & \cellcolor{popsoft}0.960 & \cellcolor{popsoft}0.920 & \cellcolor{popsoft}0.963 & \cellcolor{popsoft}0.980 \\
\cellcolor{gpqasoft}GPQA & \cellcolor{gpqasoft}50 & \cellcolor{gpqasoft}0.580 & \cellcolor{gpqasoft}0.620 & \cellcolor{gpqasoft}0.600 & \cellcolor{gpqasoft}0.960 & \cellcolor{gpqasoft}0.917 & \cellcolor{gpqasoft}0.967 & \cellcolor{gpqasoft}0.980 \\
\midrule
\textbf{Overall} &
\textbf{200} &
\textbf{0.470} &
\textbf{0.465} &
\textbf{0.460} &
\textbf{0.965} &
\textbf{0.930} &
\textbf{0.963} &
\textbf{0.975}
\\
\bottomrule
\end{tabular}%
}
\end{table*}

Taken together, the audit provides direct support for using benchmark correctness as the edit
indicator in the main experiments rather than merely matching edit prevalence by coincidence. The
benchmark proxy disagrees with the human annotator on only \(7\) of \(200\) audited interactions,
while its per-dataset edit rates stay close to the human rates and its example-level agreement is
consistently high. The independent GPT-5.6 Sol annotations lead to the same qualitative conclusion,
which further suggests that the benchmark labels capture a stable notion of substantive edit need.

\subsection{Operating-Point Trade-offs and Complementary  Metrics. }
\label{app:operating-point-tradeoffs}

The main experiments use one operating point for each model--dataset pair. To examine whether the
observed behavior is specific to that choice, we additionally sweep the Type-I error tolerance
\(\varepsilon\in\{0.1,0.2,0.3,0.4,0.5,0.6\}\) and the per-round editor budget
\(B\in\{2,3,4,5,7\}\) for \EDITREP{}. For each configuration, we repeat the online experiment over
five realized streams and report the mean together with \(95\%\) confidence intervals. The
\LLMSelect{} marker corresponds to its main \(B=5\) run and is included as a reference
operating point.

\begin{figure*}[t]
    \centering
    \includegraphics[width=\textwidth]{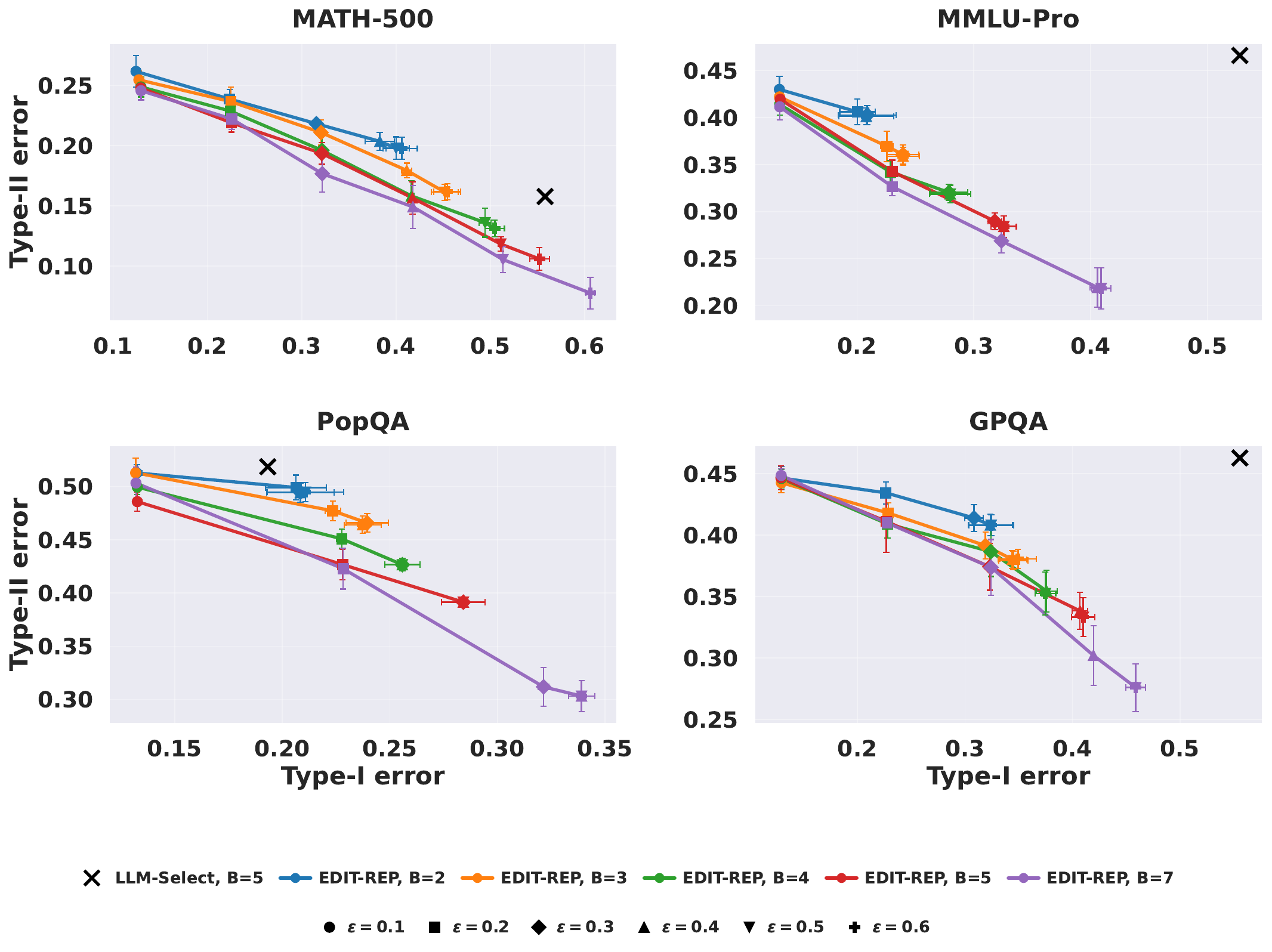}
    \caption{
    Type-I--Type-II  trade-offs for \EDITREP{}.
    Panels correspond to MATH-500, MMLU-Pro, PopQA, and GPQA. Colored curves vary the
    editor budget \(B\), while marker shapes denote the Type-I error tolerance \(\varepsilon\).
    Error bars show \(95\%\) confidence intervals across five realized streams.
    The black cross shows the existing \LLMSelect{} run with \(B=5\).
    Lower values on both axes are better.
    }
    \label{fig:pareto-type-i-type-ii}
\end{figure*}

Figure~\ref{fig:pareto-type-i-type-ii} makes the operating-point trade-off explicit. As the policy is
allowed to select more broadly, it can reduce the edit rate among skipped interactions, at the cost
of a larger confirmation rate among selected interactions. Across the four datasets, the resulting
curves show that the relative advantage of informative \EDITREP{} ranking is not tied to a single
choice of \(\varepsilon\) or \(B\). Figure~\ref{fig:pareto-cost-recall} presents the same sweep in terms of the amount of editor effort
actually used and the fraction of editable interactions discovered. Increasing the available editor budget
generally increases edit recall, while the different values of \(\varepsilon\) determine how
aggressively that budget is used.

\begin{figure*}[t]
    \centering
    \includegraphics[width=\textwidth]{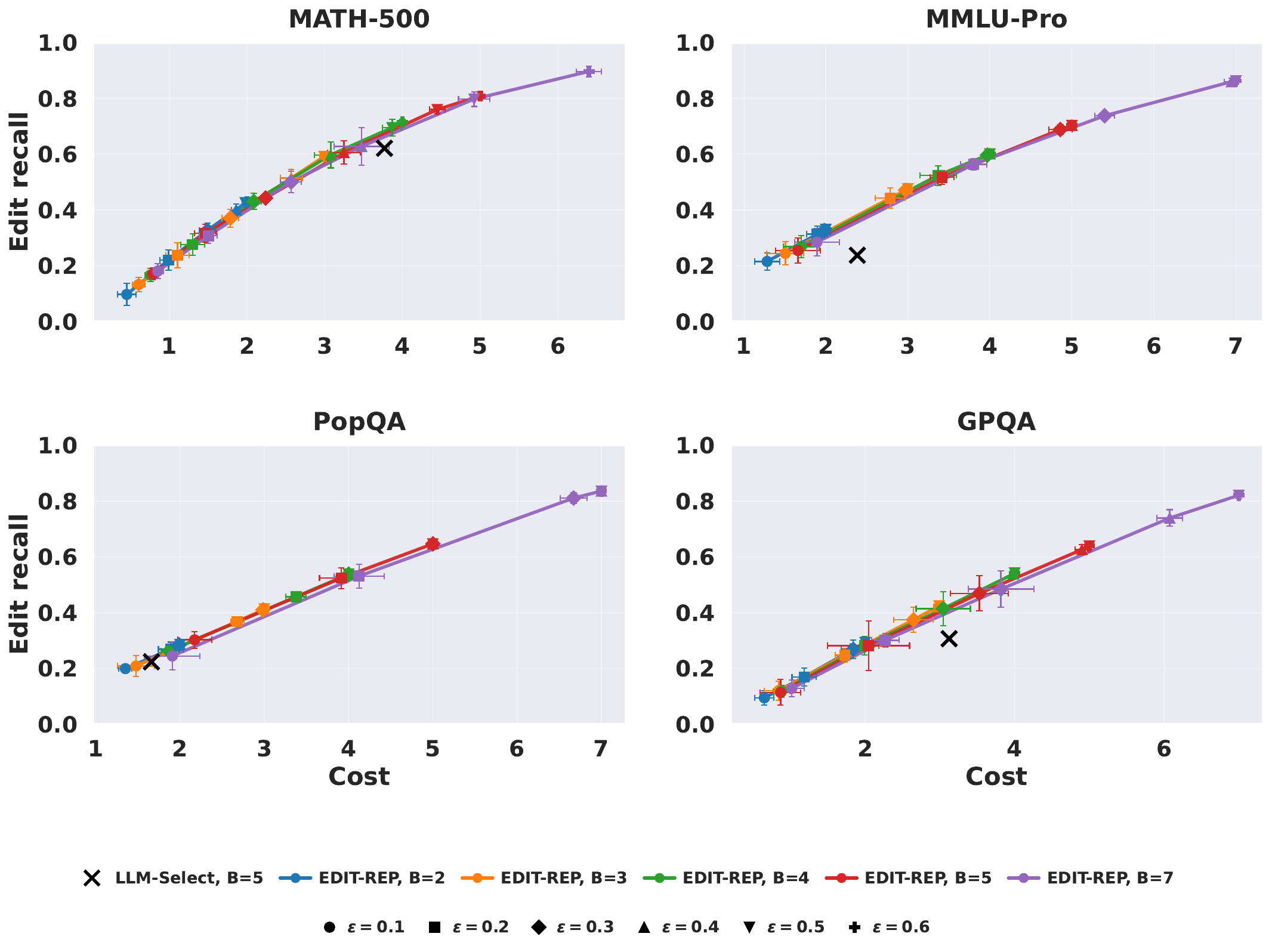}
    \caption{
    Editor-cost--edit-recall trade-offs for \EDITREP{}.
    Panels correspond to MATH-500, MMLU-Pro, PopQA, and GPQA. Colored curves vary the
    editor budget \(B\), while marker shapes denote the Type-I error tolerance \(\varepsilon\).
    Error bars show \(95\%\) confidence intervals across five realized streams.
    The black cross shows the existing \LLMSelect{} run with \(B=5\).
    Higher edit recall is better.
    }
    \label{fig:pareto-cost-recall}
\end{figure*}

To complement the conditional Type-I and Type-II errors used in the main text, Table~\ref{tab:new-discovery-metrics}
reports three discovery-centric quantities on the existing main Llama-3.3-70B-Instruct and Qwen3-32B runs:
edit recall, the absolute number of missed edits in the realized stream, and edit precision among selected
interactions. These metrics provide a complementary view of how much corrective data each policy recovers
at its realized editor operating point. Across both response models, \EDITREP{} attains the highest edit recall
and fewest missed edits on most dataset--model pairs, while \EDITLLM{} achieves the strongest edit precision
in several settings.

\begin{table*}[b]
\centering

\begin{minipage}[t]{0.315\textwidth}
\vspace{0pt}
\centering
\scriptsize

\begin{minipage}[t][4.8\baselineskip][t]{\linewidth}
\captionof{table}{
Type-I error tolerance used in the response-length-aware experiments.
}
\label{tab:length-cost-epsilon-values}
\end{minipage}

\renewcommand{\arraystretch}{1.16}
\setlength{\tabcolsep}{4pt}

\begin{tabularx}{\linewidth}{
    >{\raggedright\arraybackslash}X
    >{\centering\arraybackslash}X
}
\toprule
\textbf{Dataset} & \(\boldsymbol{\varepsilon}\) \\
\midrule
\rowcolor{mathsoft}
\mbox{MATH-500} & 0.535 \\

\rowcolor{gpqasoft}
\mbox{GPQA} & 0.485 \\
\bottomrule
\end{tabularx}

\end{minipage}%
\hfill%
\begin{minipage}[t]{0.315\textwidth}
\vspace{0pt}
\centering
\scriptsize

\begin{minipage}[t][4.8\baselineskip][t]{\linewidth}
\captionof{table}{
Response-model accuracy on the distribution-shift streams.
}
\label{tab:shift-model-accuracy}
\end{minipage}

\renewcommand{\arraystretch}{1.16}
\setlength{\tabcolsep}{4pt}

\begin{tabularx}{\linewidth}{
    >{\raggedright\arraybackslash}X
    >{\centering\arraybackslash}X
}
\toprule
\textbf{Dataset} & \textbf{Accuracy} \\
\midrule
\rowcolor{mmlusoft}
\mbox{MMLU-Pro} & 0.540 \\

\rowcolor{popsoft}
\mbox{PopQA} & 0.393 \\
\bottomrule
\end{tabularx}

\end{minipage}%
\hfill%
\begin{minipage}[t]{0.315\textwidth}
\vspace{0pt}
\centering
\scriptsize

\begin{minipage}[t][4.8\baselineskip][t]{\linewidth}
\captionof{table}{
Type-I error tolerance used in the distribution-shift experiments.
}
\label{tab:shift-epsilon-values}
\end{minipage}

\renewcommand{\arraystretch}{1.16}
\setlength{\tabcolsep}{4pt}

\begin{tabularx}{\linewidth}{
    >{\raggedright\arraybackslash}X
    >{\centering\arraybackslash}X
}
\toprule
\textbf{Dataset} & \(\boldsymbol{\varepsilon}\) \\
\midrule
\rowcolor{mmlusoft}
\mbox{MMLU-Pro} & 0.445 \\

\rowcolor{popsoft}
\mbox{PopQA} & 0.100 \\
\bottomrule
\end{tabularx}

\end{minipage}

\end{table*}

\begin{table*}[t]
\centering
\caption{
Complementary  metrics on the main Llama-3.3-70B-Instruct and Qwen3-32B runs.
Edit recall is the fraction of editable interactions that are selected; missed edits counts editable
interactions that are skipped; and edit precision is the fraction of selected interactions that require an edit.
Within each model--dataset block, the best value is shown in \textbf{bold} and the second-best is \underline{underlined}.
}
\label{tab:new-discovery-metrics}

\renewcommand{\arraystretch}{1.14}
\setlength{\tabcolsep}{5pt}
\scriptsize

\begin{tabularx}{\textwidth}{
    >{\raggedright\arraybackslash}p{1.35cm}
    >{\raggedright\arraybackslash}p{2.35cm}
    >{\raggedright\arraybackslash}p{1.55cm}
    >{\centering\arraybackslash}X
    >{\centering\arraybackslash}X
    >{\centering\arraybackslash}X
}
\toprule
\textbf{Dataset} &
\textbf{Model} &
\textbf{Method} &
\textbf{Edit recall \(\uparrow\)} &
\textbf{Missed edits \(\downarrow\)} &
\textbf{Edit precision \(\uparrow\)}
\\
\midrule

\rowcolor{mathsoft}
& & \EDITREP{} & \textbf{0.784} & \textbf{143} & \underline{0.467} \\
\rowcolor{mathsoft}
& & \EDITLLM{} & \underline{0.691} & \underline{205} & \textbf{0.473} \\
\rowcolor{mathsoft}
& & \EDITRAND{} & 0.151 & 563 & 0.253 \\
\rowcolor{mathsoft}
& \multirow{-4}{=}{\raggedright Llama-3.3-70B-Instruct}
& \LLMSelect{} & 0.621 & 251 & 0.437 \\
\rowcolor{mathsoft}
& & \EDITREP{} & \underline{0.810} & \underline{96} & \underline{0.366} \\
\rowcolor{mathsoft}
& & \EDITLLM{} & \textbf{0.939} & \textbf{31} & \textbf{0.407} \\
\rowcolor{mathsoft}
& & \EDITRAND{} & 0.190 & 410 & 0.192 \\
\rowcolor{mathsoft}
\multirow{-8}{*}{\mbox{MATH-500}}
& \multirow{-4}{=}{\raggedright Qwen3-32B}
& \LLMSelect{} & 0.765 & 119 & 0.342 \\
\midrule

\rowcolor{mmlusoft}
& & \EDITREP{} & \textbf{0.703} & \textbf{346} & \textbf{0.656} \\
\rowcolor{mmlusoft}
& & \EDITLLM{} & 0.292 & 826 & \underline{0.469} \\
\rowcolor{mmlusoft}
& & \EDITRAND{} & \underline{0.401} & \underline{699} & 0.466 \\
\rowcolor{mmlusoft}
& \multirow{-4}{=}{\raggedright Llama-3.3-70B-Instruct}
& \LLMSelect{} & 0.238 & 889 & 0.465 \\
\rowcolor{mmlusoft}
& & \EDITREP{} & \textbf{0.663} & \textbf{385} & \textbf{0.609} \\
\rowcolor{mmlusoft}
& & \EDITLLM{} & \underline{0.462} & \underline{614} & \underline{0.552} \\
\rowcolor{mmlusoft}
& & \EDITRAND{} & 0.242 & 866 & 0.455 \\
\rowcolor{mmlusoft}
\multirow{-8}{*}{\mbox{MMLU-Pro}}
& \multirow{-4}{=}{\raggedright Qwen3-32B}
& \LLMSelect{} & 0.445 & 634 & 0.530 \\
\midrule

\rowcolor{popsoft}
& & \EDITREP{} & \textbf{0.411} & \textbf{824} & 0.746 \\
\rowcolor{popsoft}
& & \EDITLLM{} & \underline{0.284} & \underline{1002} & \textbf{0.756} \\
\rowcolor{popsoft}
& & \EDITRAND{} & 0.109 & 1246 & 0.571 \\
\rowcolor{popsoft}
& \multirow{-4}{=}{\raggedright Llama-3.3-70B-Instruct}
& \LLMSelect{} & 0.223 & 1087 & \underline{0.752} \\
\rowcolor{popsoft}
& & \EDITREP{} & \textbf{0.394} & \textbf{1177} & 0.915 \\
\rowcolor{popsoft}
& & \EDITLLM{} & 0.215 & 1526 & \textbf{0.961} \\
\rowcolor{popsoft}
& & \EDITRAND{} & 0.142 & 1667 & 0.791 \\
\rowcolor{popsoft}
\multirow{-8}{*}{\mbox{PopQA}}
& \multirow{-4}{=}{\raggedright Qwen3-32B}
& \LLMSelect{} & \underline{0.260} & \underline{1438} & \underline{0.930} \\
\midrule

\rowcolor{gpqasoft}
& & \EDITREP{} & \textbf{0.644} & \textbf{407} & \textbf{0.589} \\
\rowcolor{gpqasoft}
& & \EDITLLM{} & \underline{0.615} & \underline{440} & \underline{0.576} \\
\rowcolor{gpqasoft}
& & \EDITRAND{} & 0.472 & 604 & 0.464 \\
\rowcolor{gpqasoft}
& \multirow{-4}{=}{\raggedright Llama-3.3-70B-Instruct}
& \LLMSelect{} & 0.306 & 793 & 0.448 \\
\rowcolor{gpqasoft}
& & \EDITREP{} & \textbf{0.627} & \textbf{513} & \underline{0.690} \\
\rowcolor{gpqasoft}
& & \EDITLLM{} & \underline{0.625} & \underline{516} & \textbf{0.696} \\
\rowcolor{gpqasoft}
& & \EDITRAND{} & 0.353 & 890 & 0.561 \\
\rowcolor{gpqasoft}
\multirow{-8}{*}{\mbox{GPQA}}
& \multirow{-4}{=}{\raggedright Qwen3-32B}
& \LLMSelect{} & 0.360 & 880 & 0.580 \\

\bottomrule
\end{tabularx}
\end{table*}

\subsection{List of LLM Prompts}
\label{app:prompt-list}

The experiments and audit use four prompt families. The first generates the candidate response
\(\widetilde Y\) for each benchmark question. The second is used by \LLMSelect{} to make one
budgeted selection decision for an entire batch, the third is used by \EDITLLM{} to produce
one scalar edit score for each prompt--response pair, and the fourth is the common evaluation
prompt used by GPT-5.6 Sol in the audit of benchmark-derived edit labels.

\subsubsection{Candidate-Response Generation Prompts}
\label{app:candidate-response-prompts}

\begin{datasetprompt}[title={MATH-500}]{mathblue}
\textbf{System instruction.}

{\ttfamily\small
You are a careful mathematical problem solver.

Hard requirements:\newline
- Solve the problem accurately.\newline
- End with exactly one final line in this format:\newline
\#\#\#\# <answer>

The final line must contain only the marker \#\#\#\# followed by the final answer.
}

\medskip
\textbf{User prompt.}

{\ttfamily\small
Solve the following math problem.\newline
End with exactly one final line in this format: \#\#\#\# <answer>

Problem:\newline
\{question\}
}
\end{datasetprompt}

\begin{datasetprompt}[title={MMLU-Pro}]{mmluteal}
\textbf{System instruction.}

{\ttfamily\small
You are a careful multiple-choice reasoning system.

Hard requirements:\newline
- Solve the question and choose exactly one option.\newline
- The final answer must be only the option letter.\newline
- End with exactly one final line in this format:\newline
\#\#\#\# <letter>

The final line must contain only the marker \#\#\#\# followed by a single letter such as
A, B, C, D, E, F, G, H, I, or J.
}

\medskip
\textbf{User prompt.}

{\ttfamily\small
Solve the following MMLU-Pro multiple-choice question.\newline
Choose the single best option. Think carefully if needed, but do not include reasoning in the
output.\newline
End with exactly one final line in this format: \#\#\#\# <letter>

\{question\}
}
\end{datasetprompt}

\begin{datasetprompt}[title={PopQA}]{poporange}
\textbf{System instruction.}

{\ttfamily\small
You are a careful open-domain factual question-answering system.

Hard requirements:\newline
- Answer the question accurately.\newline
- Provide only the final short answer, not a sentence or explanation.\newline
- The answer should be an entity, title, place, date, number, or short phrase.\newline
- End with exactly one final line in this format:\newline
\#\#\#\# <answer>

The final line must contain only the marker \#\#\#\# followed by the answer.
}

\medskip
\textbf{User prompt.}

{\ttfamily\small
Answer the following open-domain factual question from PopQA.\newline
Give only the final short answer; do not include explanation, reasoning, or extra text.\newline
The answer should be an entity, title, place, date, number, or short phrase.\newline
End with exactly one final line in this format: \#\#\#\# <answer>

Question:\newline
\{question\}
}
\end{datasetprompt}

\begin{datasetprompt}[title={GPQA}]{gpqapurple}
\textbf{System instruction.}

{\ttfamily\small
You are a careful graduate-level science reasoning system.

Hard requirements:\newline
- Solve the multiple-choice question and choose exactly one option.\newline
- The final answer must be only the option letter.\newline
- End with exactly one final line in this format:\newline
\#\#\#\# <letter>

The final line must contain only the marker \#\#\#\# followed by a single letter:
A, B, C, or D.
}

\medskip
\textbf{User prompt.}

{\ttfamily\small
Solve the following GPQA graduate-level science multiple-choice question.\newline
Choose the single best option. Think carefully if needed.\newline
End with exactly one final line in this format: \#\#\#\# <letter>

\{question\}
}
\end{datasetprompt}

\subsubsection{\LLMSelect{} Batch-Selection Prompts}
\label{app:llm-select-prompts}

\begin{promptouter}{Common prompt}
\textbf{System instruction.}

{\ttfamily\small
You are an expert edit-selection system. Your only task is to choose which model-generated answers
should be sent for human correction. Return only valid JSON. Do not solve the questions. Do not
include explanations.
}

\medskip
\textbf{User-prompt scaffold.}

{\ttfamily\small
You are selecting examples to send to a human editor. You will receive a batch of model-generated answers.
Select the items most likely to contain an error and therefore most worth sending to a human
corrector. Respect the total budget \{budget\}; each item has a listed cost. Select any subset whose
total cost is at most the budget. Use only the question/problem and the model answer. Do not use or
assume the gold answer. You may select zero items if none appear worth human correction. Return only
valid JSON in exactly this format: \{"selected\_indices": [0, 3, 4]\}. Do not include explanations,
markdown, or text outside the JSON object.

Batch:\newline
\{items\}
}
\end{promptouter}

\begin{datasetprompt}[title={MATH-500}]{mathblue}
{\ttfamily\small
The batch contains model-generated mathematics answers. Prefer examples with suspicious reasoning,
arithmetic mistakes, missing final answers, format violations, or unsupported conclusions. Each item
is formatted with an index, cost, problem, and model answer.
}
\end{datasetprompt}

\begin{datasetprompt}[title={MMLU-Pro}]{mmluteal}
{\ttfamily\small
Each item is a multi-option academic or professional reasoning question. The desired response is one
option letter such as A, B, C, \ldots, or J. Prefer examples with questionable reasoning, no valid
option, multiple conflicting options, unsupported certainty, or a final answer inconsistent with the
question and choices. Each item contains an index, cost, question and options, and model answer.
}
\end{datasetprompt}

\begin{datasetprompt}[title={PopQA}]{poporange}
{\ttfamily\small
The desired response is a short factual answer: an entity, title, place, date, number, or short phrase.
Prefer likely factual errors, missing answers, multiple conflicting answers, unsupported hedging,
non-short-answer formatting, or responses too vague to match a valid PopQA alias. Each item contains
an index, cost, question, and model answer.
}
\end{datasetprompt}

\begin{datasetprompt}[title={GPQA}]{gpqapurple}
{\ttfamily\small
Each item is a graduate-level science multiple-choice question with four options. The desired response
is one option letter: A, B, C, or D. Prefer questionable scientific reasoning, no valid option,
multiple conflicting options, unsupported certainty, or a final answer inconsistent with the question
and choices. Each item contains an index, cost, question and options, and model answer.
}
\end{datasetprompt}

\subsubsection{\EDITLLM{}  Edit-Score Prompts}
\label{app:edit-llm-prompts}

\begin{promptouter}{Common prompt}
\textbf{System instruction.}

{\ttfamily\small
You are an expert editor of model-generated answers. Estimate whether a human editor would need
to modify the model answer. Return only valid JSON. Do not include explanations, markdown, or extra
text.
}

\medskip
\textbf{User-prompt scaffold.}

{\ttfamily\small
Estimate the posterior edit probability for this prompt--response pair. The score is a probability in
$[0,1]$ that a human editor would modify the model answer. A high score means that correction is
likely to be required; a low score means that the answer is likely to be confirmed without
modification. Use only the question/problem and model answer below. Do not assume access to the gold
answer. Return only valid JSON in exactly this format: \{"score": 0.73\}.
}
\end{promptouter}

\begin{datasetprompt}[title={MATH-500}]{mathblue}
{\ttfamily\small
Treat mathematical errors, unsupported conclusions, missing final answers, ambiguity, and formatting
problems as evidence that human modification may be required.

Problem:\newline
\{question\}\newline
\newline
Model answer:\newline
\{model\_answer\}
}
\end{datasetprompt}

\begin{datasetprompt}[title={MMLU-Pro}]{mmluteal}
{\ttfamily\small
The expected response is one option letter such as A, B, C, \ldots, or J. Treat a likely wrong option,
an invalid option letter, ambiguity, unsupported reasoning, missing final answers, multiple conflicting
choices, or answer text without a clear option letter as evidence that correction may be required.

Question and options:\newline
\{question\}\newline
\newline
Model answer:\newline
\{model\_answer\}
}
\end{datasetprompt}

\begin{datasetprompt}[title={PopQA}]{poporange}
{\ttfamily\small
The expected response is a concise entity, title, place, date, number, or short phrase. Treat likely
factual errors, missing answers, excessive vagueness or specificity, malformed output, explanations,
multiple conflicting answers, unsupported hedging, and text unlikely to match a valid answer alias as
evidence that correction may be required.

Question:\newline
\{question\}\newline
\newline
Model answer:\newline
\{model\_answer\}
}
\end{datasetprompt}

\begin{datasetprompt}[title={GPQA}]{gpqapurple}
{\ttfamily\small
The expected response is one option letter: A, B, C, or D. Treat a likely wrong option, an invalid
letter, ambiguity, unsupported scientific reasoning, missing final answers, multiple conflicting
choices, or answer text without a clear option letter as evidence that correction may be required.

Question and options:\newline
\{question\}\newline
\newline
Model answer:\newline
\{model\_answer\}
}
\end{datasetprompt}

\subsubsection{GPT-5.6 Sol Human-Audit Prompt}
\label{app:gpt-audit-prompt}

\begin{promptouter}{Common prompt}
{\ttfamily\small
You are an independent editor of model responses. For each example, decide whether the response
would need to be modified before being accepted as a satisfactory answer to the user.

Assign:\newline
- A = 0 if the response is acceptable as-is.\newline
- A = 1 if the response genuinely requires editing before acceptance.

Judge the response based on correctness, reasoning quality, completeness, instruction following,
relevance, and clarity. Be conservative with A = 1: minor stylistic issues, verbosity, or wording
that could be improved should not count as requiring an edit if the response is still satisfactory.

Use the reference or ground-truth answer as supporting evidence, not as a strict string-matching
criterion. Accept equivalent answers, aliases, and reasonable interpretations of ambiguous questions.

For reasoning and mathematical tasks, evaluate the full response rather than only the final answer.
A correct final answer may still require editing if the reasoning contains substantive errors or
contradictions.

The key question is: ``Would you modify this response before accepting it as a satisfactory response
to the user?''
}
\end{promptouter}

\subsection{Additional Experiments with Qwen3-32B}
\label{app:qwen-experiments}

To test whether the main conclusions depend on the primary response-model family, we repeat the
main experiment with Qwen3-32B \citep{yang2025qwen3}. The response model, \EDITLLM{} scorer,
and \LLMSelect{} selector are changed consistently to Qwen3-32B; the online protocol and metric
computation remain otherwise unchanged.

Figure~\ref{fig:main-qwen-results} shows that the main qualitative trends persist with Qwen3-32B.
As in the main experiments, \EDITREP{} generally achieves lower Type-II error than
\LLMSelect{} across the four datasets, with particularly clear improvements on MMLU-Pro and GPQA.
This suggests that the explicit decomposition into an edit score, online Type-I feedback, and budget
control remains effective when the underlying response model is changed.
The comparison among the \EDIT{} variants again highlights the importance of score quality.
\EDITREP{} and \EDITLLM{} substantially outperform \EDITRAND{} on MATH-500, MMLU-Pro,
and PopQA, while their relative ordering varies across datasets. Overall, these results show that the
benefits of informative edit scores, together with the \EDIT{} controller, are not specific to the
underlying response model used in our main experiments.

\begin{figure*}[t]
    \centering
    \includegraphics[width=\textwidth]{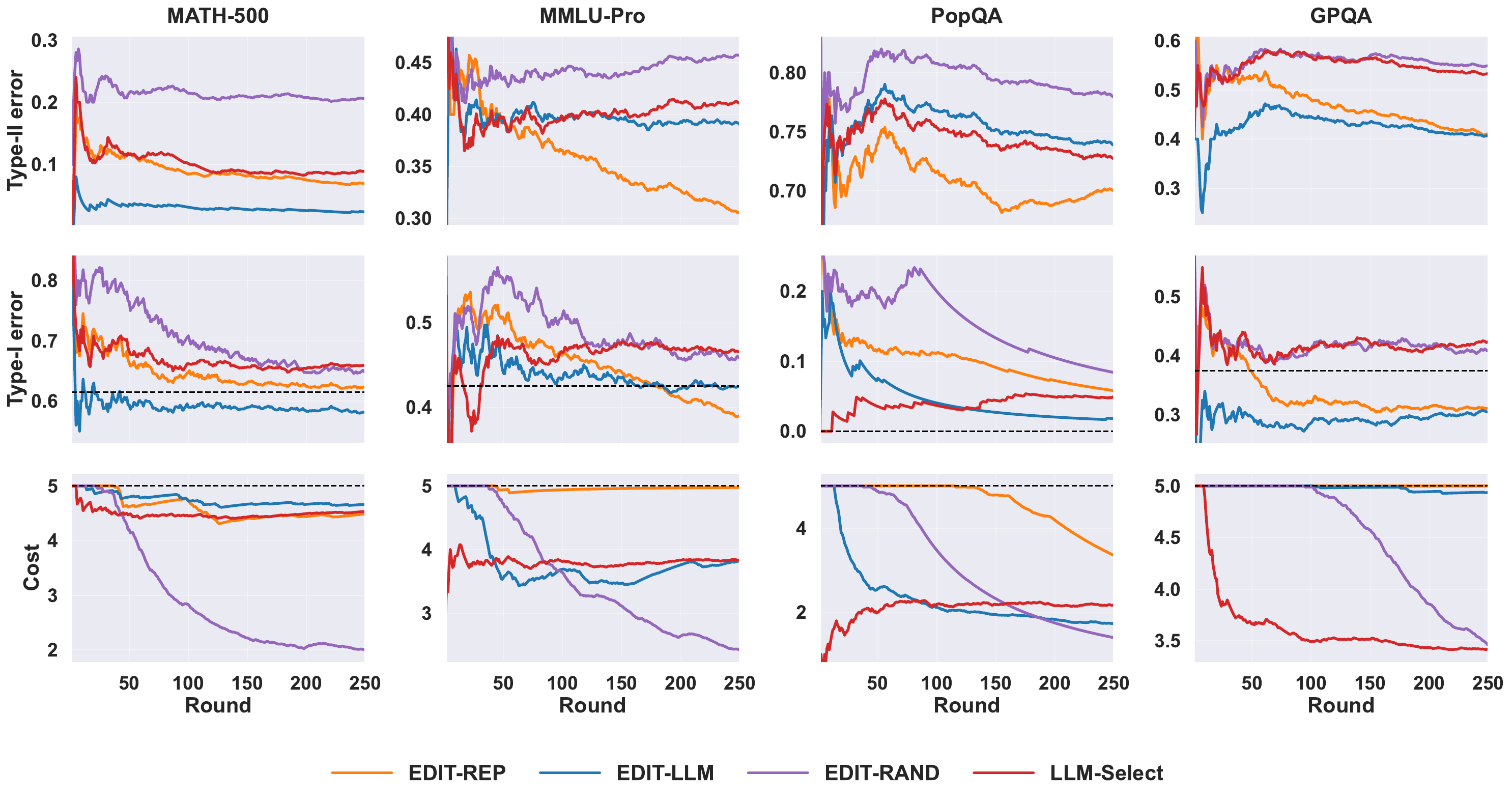}
    \caption{
    Main results with Qwen3-32B.
    Columns correspond to MATH-500, MMLU-Pro, PopQA, and GPQA; rows show running-average Type-II error,
    Type-I error, and average editor cost. Dashed lines indicate the configured Type-I error tolerance
    and per-round editor budget. 
    }
    \label{fig:main-qwen-results}
\end{figure*}

\subsection{Response-Length-Aware Editor Costs}
\label{app:length-aware-results}

Unit editor cost abstracts away an important practical asymmetry: checking a long derivation usually
requires more editor effort than validating a short factual response. We therefore define the
response-length cost
\begin{equation}
    c_{\mathrm{len}}(z)
    =
    \operatorname{clip}\!\left(
        \frac{L(\widetilde Y)}
        {\operatorname{median}_{z'}L(\widetilde Y')},
        0.25,
        4.0
    \right),
    \label{eq:review-length-cost}
\end{equation}
where \(L(\widetilde Y)\) is the response word count and the median is computed within the
corresponding generation. We retain \(B_t=5\), so the number of examples
that fit within the budget depends on their individual editor costs. The Type-I error tolerances used in
these experiments are reported in Table~\ref{tab:length-cost-epsilon-values}.

Figure~\ref{fig:review-length-cost-distribution} shows the resulting cost distributions. MATH-500
has a pronounced right tail because some responses contain long derivations, whereas PopQA is more
concentrated around unit cost, with a smaller number of long-form outliers. This variation illustrates
why assigning the same editor cost to every interaction may not accurately reflect the editor effort
required to evaluate different responses.

\begin{figure*}[t]
    \centering
    \includegraphics[width=\textwidth]{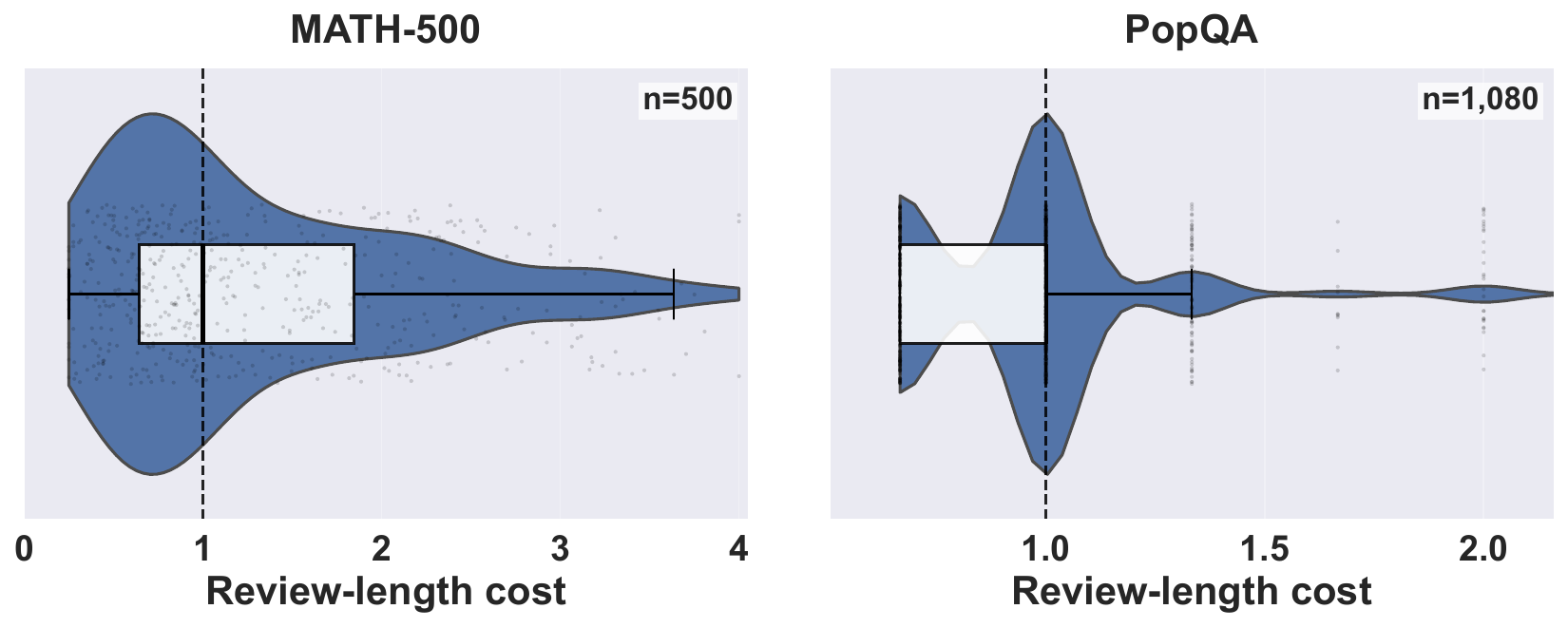}
    \caption{
    Distribution of normalized response-length editor costs.
    The violin, box, and point layers show the empirical cost distributions for MATH-500 (left) and PopQA (right).
    The dashed vertical line marks unit cost, corresponding to the median response length within the
    relevant generation; costs are clipped to the interval \([0.25,4.0]\).
    }
    \label{fig:review-length-cost-distribution}
\end{figure*}

\begin{figure*}[t]
    \centering
    \includegraphics[width=\textwidth]{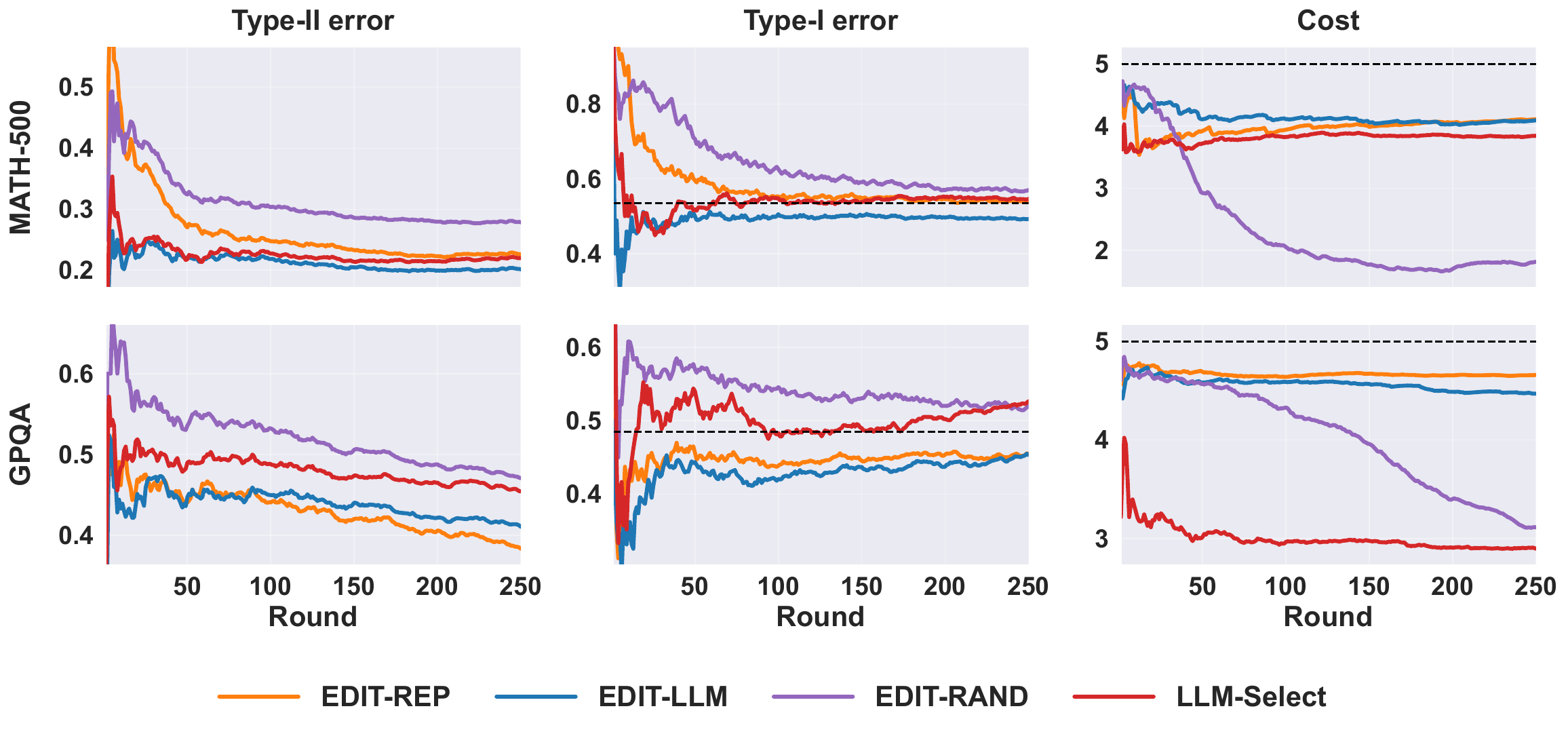}
    \caption{\textbf{Edit discovery with response-length-dependent editor cost.}
    Results use Llama-3.3-70B-Instruct on MATH-500 and GPQA with the cost in
    Equation~\eqref{eq:review-length-cost} and budget \(B_t=5\). Rows report running-average Type-II
    error, Type-I error, and average editor cost. Dashed lines denote the configured
    Type-I error tolerance and per-round editor budget.}
    \label{fig:review-length-cost-main}
\end{figure*}

Figure~\ref{fig:review-length-cost-main} shows that the cost-aware structure of \EDIT{} remains effective when
editor effort varies across interactions. On MATH-500, both semantically informed \EDIT{} variants
substantially reduce Type-II error relative to the random-score control. On GPQA, \EDITREP{} achieves
the strongest late-stage discovery performance while remaining below the displayed Type-I error tolerance.
Although \LLMSelect{} often spends less, that apparent saving comes with more missed edits. The cost
term is therefore consequential rather than cosmetic: \EDIT{} directs a fixed amount of editor resources
toward interactions with the strongest estimated correction value after accounting for editor burden.

\subsection{Additional Experiments under Distribution Shift}
\label{app:distribution-shift}

The main experiments randomly mix examples from a fixed benchmark pool. We additionally consider a
nonstationary setting in which the distribution of the stream changes over time while the policy is
not reset. For MMLU-Pro, the stream moves through \texttt{STEM}, \texttt{life/health},
\texttt{humanities/social science}, and \texttt{law/other subject} groups. For PopQA, it moves from
\texttt{head} entities to \texttt{middle-popularity} and then \texttt{long-tail} entities. All methods
receive the same sequence of phases and the same shift boundaries. The response-model accuracies and
Type-I error tolerances for these experiments are reported in
Tables~\ref{tab:shift-model-accuracy} and~\ref{tab:shift-epsilon-values}, respectively.

\begin{figure*}[t]
    \centering
    \includegraphics[width=0.9\textwidth]{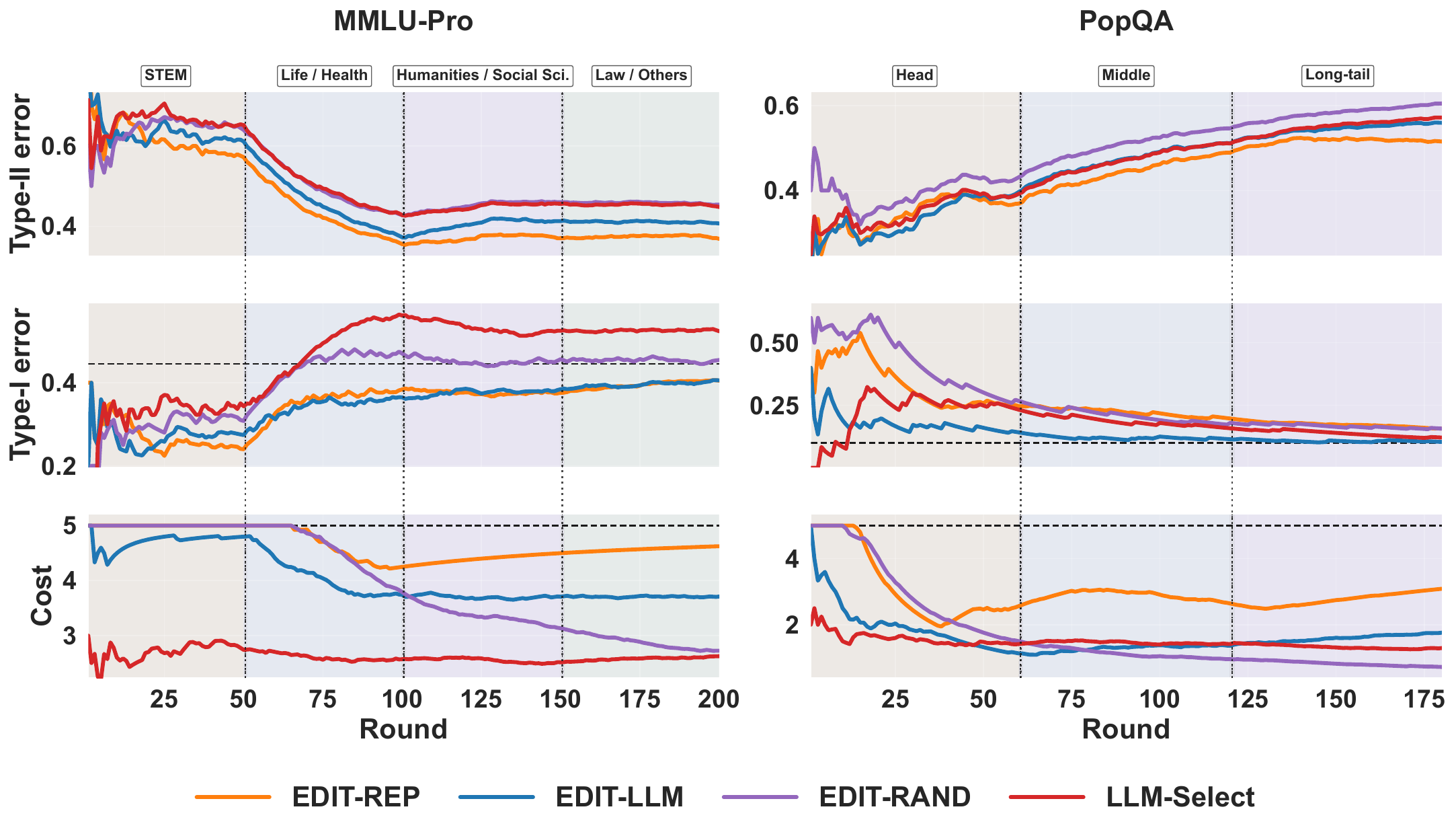}
    \caption{Robustness under structured distribution shift.
    Vertical dotted lines mark phase boundaries in the MMLU-Pro subject stream and the PopQA
    popularity stream. Rows report running-average Type-II error, Type-I error, and average editor
    cost. Dashed horizontal lines indicate
    the configured Type-I error tolerance and per-round editor budget.}
    \label{fig:distribution-shift-main}
\end{figure*}

As Figure~\ref{fig:distribution-shift-main} shows, \EDITREP{} generally attains lower Type-II error
than \LLMSelect{} while regulating Type-I error more reliably across phase changes. The separation between \EDITRAND{} and the two informed scores persists after the shifts, confirming
that adaptation of the threshold cannot compensate for an uninformative ranking. Conversely, the stable
Type-I trajectories across score variants illustrate the complementary roles predicted by the theory:
the score determines which interactions receive review, while the controller determines how
selective the policy should be.

\subsection{Effect of Embedding-Model Size}
\label{app:embedding-size}

The main experiments use a compact MiniLM encoder to represent each prompt--response pair in
\EDITREP{}. We additionally study how the capacity of the frozen representation model affects
online edit discovery. We replace the default encoder with Qwen3 models of three sizes:
Qwen3-0.6B, Qwen3-1.7B, and Qwen3-4B. In all cases, the response model, online linear probe,
selection rule, Type-I error tolerance, budget, stream order, and remaining hyperparameters are held fixed.
Thus, the only experimental difference is the model used to construct the representation
\(\phi(Z)\).

\begin{figure*}[t]
    \centering
    \includegraphics[width=0.9\textwidth]
    {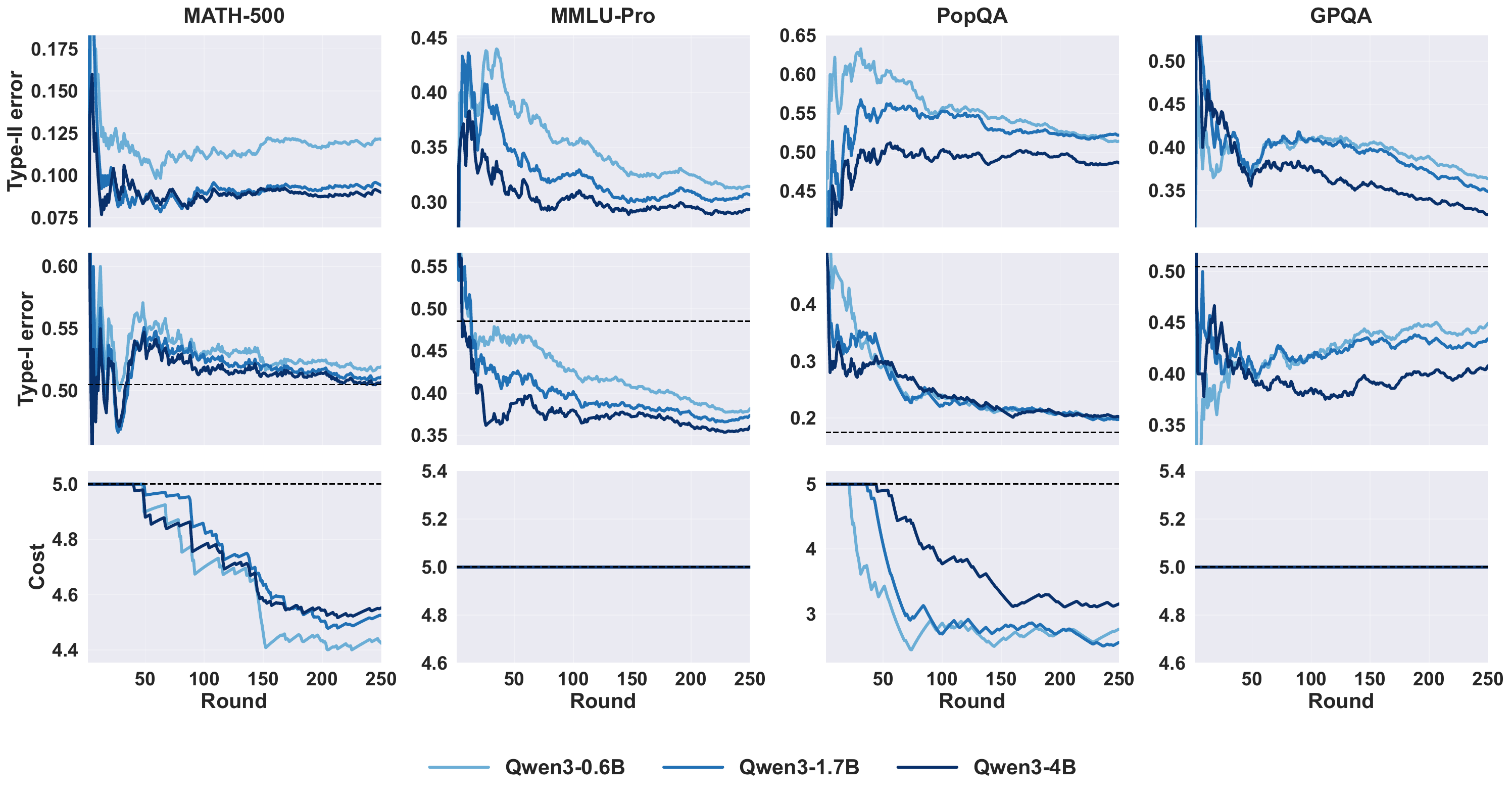}
    \caption{
    Effect of the embedding-model size on \EDITREP{}.
    Results use Llama-3.3-70B-Instruct as the response model and Qwen3-0.6B,
    Qwen3-1.7B, or Qwen3-4B as the frozen representation model. Columns correspond to
    MATH-500, MMLU-Pro, PopQA, and GPQA. Rows report running-average Type-II error,
    Type-I error, and average editor cost per round. Dashed horizontal lines indicate
    the configured Type-I error tolerance and per-round editor budget.
    }
    \label{fig:qwen-embedding-size}
\end{figure*}

Figure~\ref{fig:qwen-embedding-size} shows that increasing the embedding-model size generally
improves edit discovery. Qwen3-4B achieves the lowest final Type-II error on all four datasets,
with the clearest improvements on PopQA and GPQA. Qwen3-1.7B also consistently improves over
Qwen3-0.6B and approaches the performance of Qwen3-4B on MATH-500 and MMLU-Pro. These results
indicate that larger representation models provide features that better separate interactions likely
to require correction.

The Type-I trajectories remain  similar across encoder sizes, showing that the adaptive
threshold continues to regulate the confirmation rate independently of the representation model.
On MMLU-Pro and GPQA, all three variants use essentially the same editor budget, so the reduction
in Type-II error can be attributed directly to improved ranking rather than increased expenditure.
On MATH-500 and PopQA, the expenditure trajectories differ modestly, but the larger encoders still
provide a stronger  Type-I--Type-II trade-off. Overall, the experiment shows that embedding
capacity affects the efficiency of \EDITREP{}, while the online controller still 
regulates the edit quality and budget use.

\end{document}